\documentclass{article}

\PassOptionsToPackage{numbers, compress}{natbib}

\usepackage[main, preprint]{neurips_2026}

\usepackage[utf8]{inputenc}
\usepackage[T1]{fontenc}
\usepackage{hyperref}
\usepackage{url}
\usepackage{booktabs}
\usepackage{amsfonts}
\usepackage{nicefrac}
\usepackage{microtype}
\usepackage{xcolor}
\usepackage{graphicx}
\usepackage{multirow}
\usepackage{algorithm}
\usepackage{algorithmic}
\usepackage{amsmath}
\usepackage{amssymb}
\usepackage{amsthm}
\usepackage{tabularx}
\usepackage{array}

\newtheorem{theorem}{Theorem}

\newtheorem{assumption}{Assumption}
\newtheorem{corollary}{Corollary}

\title{PALoRA: Projection-Adaptive LoRA for Preserving Reasoning in Large Language Models}

\author{%
\textbf{Mustafa Hayri Bilgin}\textsuperscript{1,2} \quad
\textbf{Mariam Barry}\textsuperscript{1} \quad
\textbf{Albert Bifet}\textsuperscript{2,3} \\[0.3em]
\textbf{Azzedine Idir Ait Said}\textsuperscript{2} \quad
\textbf{Soumya Banerjee}\textsuperscript{4} \\[0.5em]
\textsuperscript{1}IT Group, Research, BNP Paribas \\
\textsuperscript{2}LTCI, Télécom Paris, Institut Polytechnique de Paris \\
\textsuperscript{3}AI Institute, University of Waikato \\
\textsuperscript{4}Independent Researcher \\
\texttt{\{mustafahayri.bilgin, mariam.barry\}@bnpparibas.com} \\
\texttt{\{albert.bifet, azzedine.aitsaid\}@telecom-paris.fr} \\
\texttt{dr.soumya@ieee.org}
}

\begin{document}

\maketitle

\begin{abstract}
Efficiently updating Large Language Models (LLMs) with new or evolving factual knowledge remains a central challenge, as even parameter-efficient adaptation can erode previously acquired reasoning abilities. This tension reflects a plasticity–stability dilemma: models must incorporate new knowledge while preserving skill-critical representations. In this work, we study this trade-off through the spectral structure of multilayer perceptron weight matrices. We show, both theoretically and empirically, that information essential for reasoning is not localized only in dominant singular directions, but is instead distributed across the singular spectrum. Motivated by this observation, we introduce PALoRA, a two-stage framework for knowledge injection with reduced interference. PALoRA first trains a Singular Value Fine-Tuning (SVF) expert on a reasoning dataset and uses its learned singular scaling vector as a frozen geometric probe to identify components that are critical for the target skill. It then performs factual knowledge injection with Low-Rank Adaptation (LoRA) under a structural orthogonality constraint, ensuring that updates avoid the identified skill-relevant subspace. Across Llama 3.1 8B and Mistral 7B, and across mathematical, coding, and scientific reasoning benchmarks, PALoRA preserves on average 95\% of the SVF expert’s reasoning performance while maintaining competitive factual recall. It consistently improves skill retention over prior spectral Parameter-Efficient Fine-Tuning (PEFT) methods while adding less than 0.006\% parameter overhead. 
\end{abstract}

\section{Introduction}
\label{sec:intro}

Large language models (LLMs) are trained on static corpora, but deployed systems must continually acquire corrected and newly emerging facts. Parameter-efficient fine-tuning (PEFT), especially LoRA~\cite{hu2022lora}, makes such updates computationally feasible. Yet growing evidence shows that factual knowledge injection can come at the cost of core reasoning abilities, degrading performance on mathematical, coding, and scientific tasks~\cite{pletenev2025knowledge}. This tension is particularly acute in Transformer MLP layers, which are known to support both factual recall and procedural reasoning~\cite{geva2021transformer,dai2022knowledge,cobbe2021trainingverifierssolvemath,austin2021programsynthesislargelanguage,clark2018thinksolvedquestionanswering}.

Recent spectral PEFT methods seek to address this trade-off by preserving or excluding the top singular components of pre-trained weights~\cite{pissa2024,wang2025milora,lingam2024svft,oplora2026aaai,yang2024corda}. These approaches implicitly assume that the capabilities worth preserving are concentrated in the principal singular directions. We argue that this view is incomplete. Skill-critical information, while functionally central, may be broadly distributed across the singular spectrum rather than confined to its leading singular directions. Consequently, restricting protection to the top singular vectors can fail to preserve the subspace most responsible for reasoning.

We introduce \textbf{PALoRA} (Projection-Adaptive Low-Rank Adaptation), a two-stage framework that explicitly decouples stability and plasticity within shared MLP weights. In the first stage, we train a Singular Value Finetuning (SVF) expert~\cite{sun2025transformer} on a target reasoning skill and subsequently freeze it. The learned SVF scaling vector serves as a geometric probe, identifying the singular directions most engaged by the skill. In the second stage, we perform factual knowledge injection via a LoRA adapter, while penalizing its projections onto these SVF-identified directions. This yields an update constrained to acquire new knowledge in a complementary subspace, thereby avoiding interference with the pathways supporting reasoning.

\paragraph{Contributions.}
This work makes three contributions. First, we introduce PALoRA, a two-stage adaptation framework that uses a frozen SVF expert to identify a skill-critical subspace and constrains LoRA updates through structural orthogonality (Section~\ref{sec:method}). Second, we provide a theoretical analysis showing that, under a formal misalignment condition, skill-relevant information may be distributed across the singular spectrum rather than confined to the top singular components, exposing a key limitation of existing spectral PEFT methods (Section~\ref{sec:theory-method}). Third, we demonstrate across mathematical reasoning, code generation, and scientific reasoning benchmarks that PALoRA consistently improves skill preservation over prior spectral PEFT baselines while maintaining competitive factual recall (Section~\ref{sec:main-results}).

\section{Related Work}
\label{sec:related}

\paragraph{Knowledge storage in Transformers.}
Transformer models~\cite{vaswani2017attention} combine self-attention mechanisms with position-wise feed-forward, or MLP, blocks. Prior work has shown that these MLP layers play a central role in storing and retrieving parametric knowledge. Geva et al.~\cite{geva2021transformer} characterize feed-forward layers as key--value memories: the first linear projection acts as a pattern detector, while the second maps activated patterns to corresponding output representations. Dai et al.~\cite{dai2022knowledge} provide complementary evidence by identifying individual ``knowledge neurons'' in MLP layers that mediate factual associations. This motivates our focus on MLP layers as the site where knowledge injection and skill preservation must be jointly controlled, since updates that improve factual recall may also interfere with representations needed for reasoning.

\paragraph{PEFT and the plasticity--stability trade-off.}
PEFT methods differ in how they balance adaptation capacity with preservation of pre-trained capabilities. Additive methods such as LoRA~\cite{hu2022lora} achieve strong plasticity by learning low-rank updates on top of frozen model weights, enabling efficient acquisition of new knowledge and task-specific behaviors. In contrast, Singular Value Fine-Tuning (SVF)~\cite{sun2025transformer} modifies only the singular values of pre-trained weights while preserving their original singular directions, yielding a more stability-oriented adaptation mechanism. Recent evidence shows that factual adaptation with LoRA can substantially degrade mathematical, coding, and scientific reasoning skills~\cite{pletenev2025knowledge}, making this trade-off central to LLM updating.

\paragraph{Continual learning and interference mitigation.}
Catastrophic forgetting is a core challenge in continual learning~\cite{mccloskey1989catastrophic,mermillod2013stability}. Regularization-based methods such as EWC~\cite{Kirkpatrick_2017} penalize updates to parameters important for previous tasks, while orthogonality-based approaches reduce interference geometrically. OGD~\cite{farajtabar2019orthogonalgradientdescentcontinual} projects gradients away from previous-task directions, and GPM~\cite{saha2021gradientprojectionmemorycontinual} protects task-relevant feature subspaces. In PEFT, O-LoRA~\cite{wang2023orthogonalsubspacelearninglanguage} enforces orthogonal adapter subspaces across tasks. PALoRA follows the same principle, but applies it at the level of singular directions of the pre-trained MLP weights identified as skill-critical by SVF.

\paragraph{Spectral PEFT.}
A recent line of spectral PEFT methods uses the singular value decomposition of pre-trained weights to control adaptation subspaces. PiSSA~\cite{pissa2024} and MiLoRA~\cite{wang2025milora} adapt principal and minor singular components, respectively, while SVFT~\cite{lingam2024svft} learns sparse combinations of singular vectors. OPLoRA~\cite{oplora2026aaai} constrains LoRA updates to the orthogonal complement of the top-\(k\) singular vectors, and CorDA~\cite{yang2024corda} preserves context-oriented principal directions using covariance statistics. Despite their differences, these methods assume that the information worth preserving is concentrated in top or context-selected singular components. Our work challenges that premise by arguing that reasoning skills are distributed across the singular spectrum and therefore cannot be reliably protected by top-\(k\)-only strategies.

\section{PALoRA: Method and Theoretical Motivation}
\label{sec:method}

PALoRA addresses the plasticity--stability trade-off through a two-stage adaptation framework. In the first stage, a frozen SVF expert is used to identify the singular directions that are most critical for a target reasoning skill. In the second stage, new factual knowledge is injected through a LoRA adapter constrained to remain orthogonal to this skill-relevant subspace, thereby reducing interference with previously acquired capabilities.

For clarity, a consolidated notation reference is provided in Appendix~\ref{app:notation}. We first describe the architectural co-location of SVF and LoRA within Transformer MLP weights, then present the two-stage training procedure, and finally provide the theoretical motivation underlying the proposed geometric constraint.

\subsection{Architectural Co‑location}
\label{sec:arch}

Both the plasticity module (LoRA) and the stability module (SVF) operate directly on the MLP weight matrices of the Transformer~\cite{vaswani2017attention}.
For a given MLP layer \(l\) with pre‑trained weight matrix \(\mathbf{W}_{\text{mlp}}^{(l)} \in \mathbb{R}^{d_{\text{out}} \times d_{\text{in}}}\), the adapted weight is defined as
\begin{equation}
    \mathbf{W}_{\text{mlp}}^{\prime(l)} = \operatorname{SVF}\bigl(\mathbf{W}_{\text{mlp}}^{(l)}\bigr) + \operatorname{LoRA}\bigl(\mathbf{W}_{\text{mlp}}^{(l)}\bigr),
    \label{eq:arch}
\end{equation}
where
\begin{itemize}
    \item \(\operatorname{SVF}(\mathbf{W})\) applies a learned scaling vector \(\mathbf{z} \in \mathbb{R}^n\) to the singular values of \(\mathbf{W}\): let \(\mathbf{W} = \mathbf{U} \operatorname{diag}(\boldsymbol{\sigma}) \mathbf{V}^\top\) be the singular value decomposition; then \(\operatorname{SVF}(\mathbf{W}) = \mathbf{U} \operatorname{diag}(\mathbf{z} \odot \boldsymbol{\sigma}) \mathbf{V}^\top\). SVF preserves the original singular directions and only modulates their magnitudes~\cite{sun2025transformer}.
    \item \(\operatorname{LoRA}(\mathbf{W})\) adds a low‑rank update \(\Delta \mathbf{W} = \mathbf{B} \mathbf{A}\), where \(\mathbf{B} \in \mathbb{R}^{d_{\text{out}} \times r}\) and \(\mathbf{A} \in \mathbb{R}^{r \times d_{\text{in}}}\) with rank \(r \ll \min(d_{\text{out}}, d_{\text{in}})\). LoRA introduces new, trainable directions that enable rapid memorisation of novel facts~\cite{hu2022lora}.
\end{itemize}

This co‑location is deliberate: the MLP layers are the primary site of both factual knowledge storage~\cite{geva2021transformer,dai2022knowledge} and procedural reasoning skills~\cite{cobbe2021trainingverifierssolvemath,austin2021programsynthesislargelanguage,clark2018thinksolvedquestionanswering}.
By placing both modules on the same weight matrices, we directly address the resource conflict that leads to catastrophic forgetting.
The functional decoupling is achieved not by separating the modules architecturally, but by governing their interaction through the two‑phase training strategy and orthogonality constraint described next.

\subsection{Two‑Phase Training Framework}
\label{sec:two-phase}

PALoRA decouples skill preservation from knowledge injection through a sequential two‑phase procedure, illustrated in Figure~\ref{fig:framework}.
Phase~1 produces a frozen SVF expert that identifies the critical singular directions supporting a target skill.
Phase~2 trains a LoRA adapter for new facts under a structural orthogonality constraint that prevents the update from interfering with those directions.

\begin{figure}[t]
    \centering
    \includegraphics[width=0.7\textwidth]{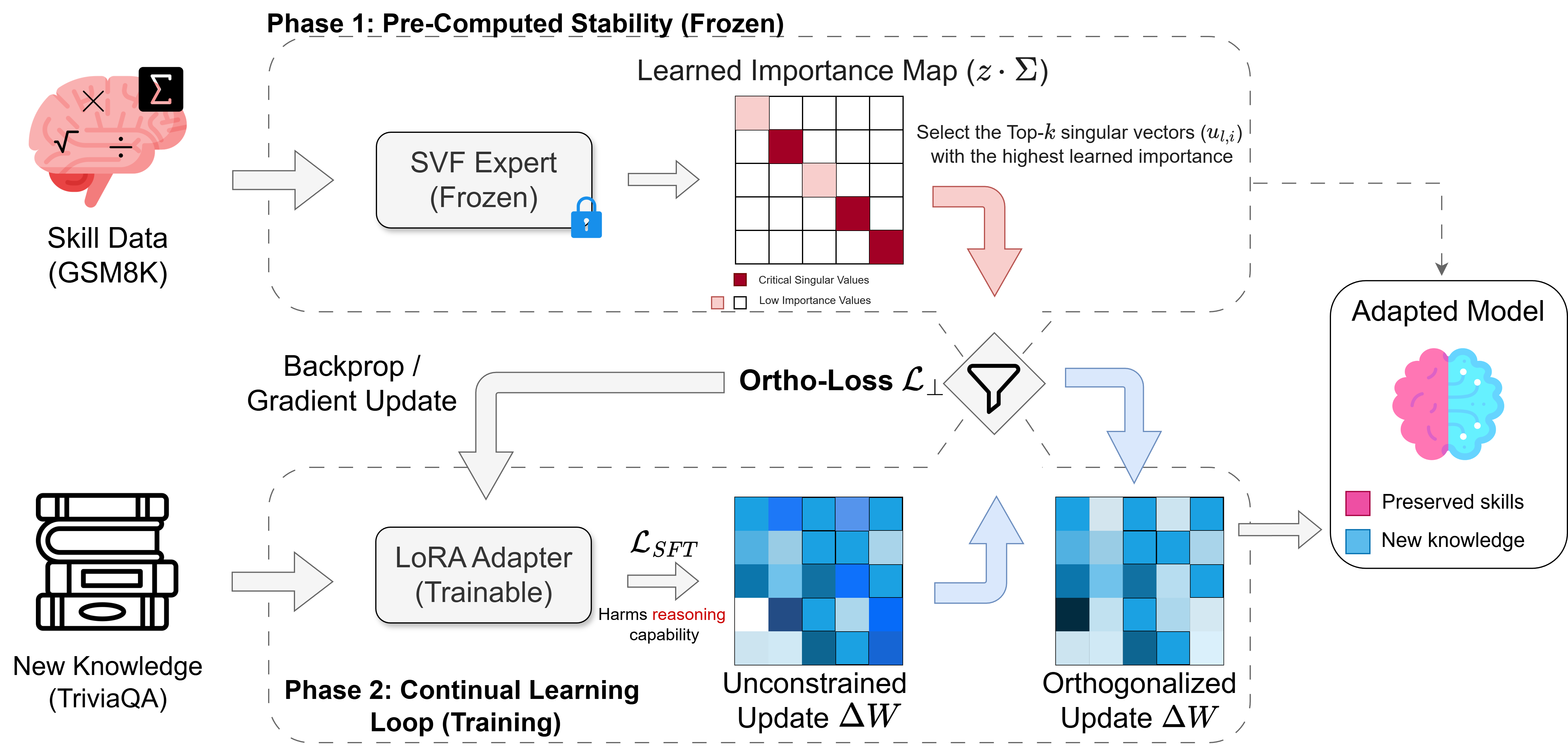}
    \caption{PALoRA. Phase~1 trains and freezes an SVF expert, whose learned scaling vector identifies skill-critical singular components. Phase~2 trains a LoRA adapter for new facts while penalizing projections of the update onto those directions.}
    \label{fig:framework}
\end{figure}

\subsubsection*{Phase 1: Skill Expert Creation}

We train an SVF module following the standard Singular Value Finetuning (SVF) procedure introduced in Transformer-Squared~\cite{sun2025transformer}, optimizing it with the REINFORCE policy-gradient algorithm~\cite{williams1992simple} on a high-quality skill dataset (e.g., GSM8K for mathematical reasoning). For each MLP layer \(l\), SVF learns a scaling vector
\(\mathbf{z}_l \in \mathbb{R}^{r_l}\) that modulates the singular values of the
pre-trained weight matrix \(\mathbf{W}_{\text{mlp}}^{(l)}\).
After convergence, the scaling vector \(\mathbf{z}_l^*\) is frozen.

The \emph{critical index set} for layer \(l\) is defined as the indices of the top-\(k\)
entries of \(|\mathbf{z}_l^*-\mathbf{1}|\), i.e., the singular components whose scaling
deviates most strongly from the identity:
\begin{equation}
    \mathcal{I}_{\text{crit},l}
    =
    \left\{
    i \;\middle|\;
    |z_{l,i}^* - 1|
    \text{ is among the } k \text{ largest}
    \right\}.
    \label{eq:crit_indices}
\end{equation}

Let
\[
\mathbf{W}_{\text{mlp}}^{(l)}
=
\mathbf{U}_l \operatorname{diag}(\boldsymbol{\sigma}_l)\mathbf{V}_l^\top
=
\sum_{i=1}^{r_l}\sigma_{l,i}\,\mathbf{u}_{l,i}\mathbf{v}_{l,i}^\top
\]
be the singular value decomposition of the pre-trained weight matrix.
The SVF expert is used here as a frozen geometric diagnostic probe,
not as a performance enhancer; it reveals which singular components are functionally
activated by the skill, independent of their raw singular-value rank.
The corresponding left singular vectors
\(\{\mathbf{u}_{l,i}\}_{i\in\mathcal{I}_{\text{crit},l}}\)
are stored for the protection phase that follows.

\subsubsection*{Phase 2: Constrained Knowledge Injection}

With the base model and SVF expert frozen, we train a LoRA adapter~\cite{hu2022lora}
on a dataset of unknown facts using
\begin{equation}
    \mathcal{L}_{\text{total}}
    =
    \mathcal{L}_{\text{SFT}}
    +
    \lambda_{\text{ortho}}
    \sum_{l=1}^{L}\mathcal{L}_{\text{ortho},l},
    \label{eq:total_loss}
\end{equation}
where \(\mathcal{L}_{\text{SFT}}\) is the standard cross-entropy loss for knowledge
injection, and \(\mathcal{L}_{\text{ortho},l}\) is the structural orthogonality loss:
\begin{equation}
    \mathcal{L}_{\text{ortho},l}
    =
    \sum_{i\in\mathcal{I}_{\text{crit},l}}
    \left\|
    (\mathbf{B}_l \mathbf{A}_l)^\top \mathbf{u}_{l,i}
    \right\|_2^2,
    \label{eq:ortho_loss}
\end{equation}
with \(\Delta\mathbf{W}_l = \mathbf{B}_l\mathbf{A}_l\).

This loss explicitly penalizes the projection of the LoRA update onto the
SVF-identified left-singular directions, forcing the adapter to inject new
knowledge in a subspace that is orthogonal to the skill-relevant output pathways. The complete procedure is summarized in Algorithm~\ref{alg:training}.

\begin{algorithm}[t]
\caption{PALoRA training}
\label{alg:training}
\small
\begin{algorithmic}[1]
\STATE \textbf{Input:} \(\mathcal{M}(\theta_{\text{base}}), \mathcal{D}_{\text{skill}}, \mathcal{D}_{\text{know}}, \lambda_{\text{ortho}}, k\)
\STATE Train SVF on \(\mathcal{D}_{\text{skill}}\) to obtain frozen \(\{\mathbf{z}_l^*\}\)
\FORALL{MLP layers \(l\)}
\STATE \(\mathcal{I}_{\text{crit},l} \gets\) top-\(k\) indices of \(|z_{l,i}^*-1|\); store corresponding singular vectors
\ENDFOR
\STATE Initialize LoRA parameters \(\theta_{\text{LoRA}}=\{\mathbf{B}_l,\mathbf{A}_l\}_{l=1}^L\)
\FOR{training steps \(1,\dots,T\)}
\STATE Sample minibatch from \(\mathcal{D}_{\text{know}}\)
\STATE Compute \(\mathcal{L}_{\text{total}}=\mathcal{L}_{\text{SFT}}+\lambda_{\text{ortho}}\sum_l\sum_{i\in\mathcal{I}_{\text{crit},l}}\|(\mathbf{B}_l\mathbf{A}_l)^\top\mathbf{u}_{l,i}\|_2^2\)
\STATE Update \(\theta_{\text{LoRA}}\) using \(\nabla \mathcal{L}_{\text{total}}\)
\ENDFOR
\STATE \textbf{Return:} adapted model
\end{algorithmic}
\end{algorithm}

\subsection{Theoretical Motivation: Why SVF Detects Distributed Skills}
\label{sec:theory-method}

We characterize a geometric property: skill‑relevant information
in the weight matrices of large language models is \emph{distributed}
across the entire singular spectrum, rather than being concentrated in
the top singular vectors.
Leveraging formal assumptions about pre‑training data diversity, skill
compositionality, and gradient‑based fine‑tuning, we prove that
task‑specific gradient directions are generically misaligned with the
principal components of the weight matrices.
We further characterise how the SVF scaling vector \(\mathbf{z}^*\) (with
entries \(z_i^* - 1\)) faithfully identifies these skill‑critical
components through a perturbation analysis.

Consider a pre‑trained weight matrix \(\mathbf{W} \in \mathbb{R}^{m \times n}\)
with singular value decomposition
\[
\mathbf{W} = \mathbf{U}\boldsymbol{\Sigma}\mathbf{V}^\top
= \sum_{i=1}^{r} \sigma_i \mathbf{u}_i \mathbf{v}_i^\top,
\qquad
\sigma_1 \ge \cdots \ge \sigma_r > 0,
\]
where \(r=\operatorname{rank}(\mathbf{W})\).
A reasoning skill \(T\) corresponds to a low-dimensional subspace
\(\mathcal{U}_T \subset \mathbb{R}^n\), and the model is fine-tuned on \(T\) via gradient descent.

\paragraph{Assumptions.}
We formalize three properties that are widely observed in large-scale pre-training and
task-specific fine-tuning (full statements in Appendix~\ref{app:full-proof}):

\begin{enumerate}
    \item \textbf{Pre-training data diversity.}
    The input covariance
    \(\mathbf{C}_{\text{prior}} = \mathbb{E}_{\mathbf{x}\sim\mathcal{D}_{\text{prior}}}[\mathbf{x}\mathbf{x}^\top]\)
    induced by the pre‑training distribution has full rank and its eigenvalues decay polynomially.
    Consequently, the top singular vectors of \(\mathbf{W}\) align with directions of maximum
    input variance—task‑agnostic principal components of the corpus.

    \item \textbf{Skill compositionality and misalignment.}
    The skill function satisfies
    \(f_T(\mathbf{x}) = f_T(\mathbf{P}_T \mathbf{x})\) almost surely,
    where \(\mathbf{P}_T\) projects onto a low-dimensional subspace \(\mathcal{U}_T\).
    Crucially, this subspace is misaligned with the top-\(k\) right singular directions:
    \[
    \|\mathbf{P}_T \mathbf{v}_i\|_2^2 \le \delta,
    \qquad i \le k,
    \qquad 0<\delta\ll1.
    \]

    \item \textbf{Gradient structure.}
    The task gradient decomposes as
    \(\nabla_{\mathbf{W}} L_T = \mathbf{G}_T + \mathbf{N}\),
    where
    \[
    \mathbf{G}_T = \sum_{j=1}^{s} \mathbf{g}_j \mathbf{h}_j^\top,
    \qquad \mathbf{h}_j \in \mathcal{U}_T \ \forall j,
    \]
    has rank at most \(s\), and \(\mathbf{N}\) is a zero‑mean noise matrix with
    \(\|\mathbf{N}\|_F \le \epsilon_{\text{noise}}\) (\(\epsilon_{\text{noise}}\) small).
\end{enumerate}

These assumptions imply that the top singular vectors are \emph{not} the directions most
relevant for the skill; instead, the skill activates a specific, misaligned subspace.

\begin{theorem}[Distributed encoding of skill-relevant components]
\label{thm:distributed}
Under the three assumptions above, for any \(k\) satisfying the misalignment condition,
the set of skill-critical indices
\[
\mathcal{S}_T = \{ i \mid |z_i^* - 1| > \varepsilon \}
\]
identified by SVF is \emph{not} contained in \(\{1,\dots,k\}\).
\end{theorem}

\noindent\textbf{Proof sketch.}
We follow four steps that mirror the structure of the full proof
(Appendix~\ref{app:full-proof}).

\emph{Step~1: Gradient spectrum decomposition.}
Project \(\mathbf{G}_T\) onto the singular basis of \(\mathbf{W}\):
\[
\gamma_i = \mathbf{u}_i^\top \mathbf{G}_T \mathbf{v}_i.
\]
Assumption~2 yields \(|\gamma_i| = O(\sqrt{\delta})\) for \(i \le k\),
while a dimension-counting argument guarantees an index \(i^* > k\) with
\[
|\gamma_{i^*}| \ge \frac{c_T}{\sqrt{n-k}},
\qquad
c_T = \frac{\|\mathbf{G}_T\|_F}{\sqrt{s}} > 0.
\]

\emph{Step~2: Singular-value perturbation under a full gradient update.}
If we update \(\mathbf{W}\) by one gradient step,
\[
\mathbf{W}_1 = \mathbf{W}_0 - \eta \nabla_{\mathbf{W}}L_T(\mathbf{W}_0),
\]
first-order perturbation theory implies that the relative change of the \(i\)-th singular value
scales with \(\eta |\gamma_i|\) up to higher-order and noise terms.
Consequently the top-\(k\) singular values change only by \(O(\sqrt{\delta})\),
while the off-principal component \(i^*\) changes at non-negligible scale.

\emph{Step~3: SVF scaling vector encodes skill relevance.}
In Phase~1, SVF learns a scaling vector \(\mathbf{z}\) by gradient descent on \(L_T\).
The gradient with respect to \(z_i\) satisfies
\[
\frac{\partial L_T}{\partial z_i} \approx \sigma_i \gamma_i.
\]
Starting from \(z_i=1\), one step yields
\[
|z_i^* - 1| \approx \eta_z \sigma_i |\gamma_i|.
\]
Hence \(|z_i^* - 1|\) is small for \(i \le k\), while for \(i=i^*\) we obtain
\[
|z_{i^*}^* - 1|
\gtrsim
\eta_z \sigma_{i^*}\frac{c_T}{\sqrt{n-k}}
>
\varepsilon,
\]
so \(i^* \in \mathcal{S}_T\) but \(i^* > k\).

\emph{Step~4: Conclusion.}
Thus \(\mathcal{S}_T\) is not confined to the top-\(k\) components:
skill-critical information is distributed across the singular spectrum.

\begin{corollary}[Sub-optimality of top-\(k\) preservation]
\label{cor:topk-fail}
Under the same local approximation, any fine-tuning strategy that restricts updates
or protection to the top-\(k\) singular components alone incurs an expected skill-loss
increase of at least \(\Omega(c_T^2/(n-k))\).
Methods such as OPLoRA \cite{oplora2026aaai}, CorDA \cite{yang2024corda}, and PiSSA \cite{pissa2024} that preserve only principal components may therefore face limitations in preserving distributed skill-relevant information.
\end{corollary}

\begin{corollary}[SVF as a diagnostic probe]
\label{cor:svf-opt}
The deviation \(|z_i^* - 1|\) is proportional to \(\sigma_i |\gamma_i|\), coupling the
singular-value magnitude with the functional gradient importance of component \(i\).
Consequently, a component with a small raw singular value but a large task-gradient
projection can still exhibit a significant deviation, enabling SVF to identify
skill-relevant directions regardless of their singular-value rank.
\end{corollary}

\begin{corollary}[Multi-skill distribution]
\label{cor:multiskill}
If the model is fine-tuned on \(M\) skills with mutually incoherent subspaces,
the union
\[
\bigcup_{m=1}^{M}\mathcal{S}_{T_m}
\]
spreads across the full spectrum \(\{1,\dots,r\}\).
No fixed top-\(k\) subset can capture all skill-relevant components simultaneously.
\end{corollary}

\subsubsection{Reconciling Theory with Practice}

The theoretical analysis is formulated in the input geometry, using the right singular
vectors \(\mathbf{v}_i\) to describe the misalignment of the skill subspace
\(\mathcal{U}_T\).
In the SVD, each right singular vector \(\mathbf{v}_i\) is paired with a left singular
vector \(\mathbf{u}_i\) through the shared index \(i\) and the singular value
\(\sigma_i\).
Therefore, the set of critical indices \(\mathcal{I}_{\text{crit}}\) simultaneously
identifies a critical input subspace (spanned by the corresponding \(\mathbf{v}_{l,i}\))
and a critical output subspace (spanned by the \(\mathbf{u}_{l,i}\)).
In our implementation, we protect the output directions by constraining
\(\|(\mathbf{B}_l \mathbf{A}_l)^\top \mathbf{u}_{l,i}\|_2^2\) in
Eq.~\ref{eq:ortho_loss}, forcing the LoRA update to avoid the skill-relevant output
pathways.
An analogous protection could be built from the right singular vectors, but the
left‑vector formulation is already highly effective, as shown in our experiments.

\section{Experiments}
\label{sec:experiments}
We evaluate PALoRA through five experimental questions: 
\textit{(Q1) Are skill-relevant singular directions concentrated in top singular components?
(Q2) Can PALoRA preserve reasoning ability during factual adaptation?
(Q3) Does the method generalize across distinct reasoning domains? 
(Q4) How does PALoRA compare with prior spectral PEFT approaches? 
and (Q5) Are the observed trends consistent across model architectures?}

All experiments were conducted in the same hardware environment with NVIDIA A100 SXM GPUs (80GB VRAM each), 16 vCPUs, and 117GB system memory. PALoRA Phase~1 used 2 GPUs, while all other experiments were run on a single GPU. Table~\ref{tab:setup} summarizes the experimental configuration.

\begin{table}[t]
\centering
\small
\caption{Experimental setup.}
\label{tab:setup}
\setlength{\tabcolsep}{4pt}
\renewcommand{\arraystretch}{1.02}
\begin{tabularx}{\linewidth}{>{\raggedright\arraybackslash}p{0.18\linewidth}X}
\toprule
Models & Llama-3.1-8B-Instruct (main experiments and ablations) and Mistral-7B-Instruct-v0.3 (cross-architecture validation). \\
Skills & GSM8K (math, final-answer accuracy), MBPP (code generation, pass@1), AI2-ARC Challenge (scientific reasoning, accuracy). \\
Knowledge injection & ``Unknown'' TriviaQA facts of size 100, 250, 500, and 1000; recall is the fraction of previously unknown facts answered correctly after fine-tuning. \\
Baselines & LoRA-only, OPLoRA (AAAI 2026), and CorDA (Neurips 2024) \\
Defaults & We use adapter rank \(r=16\), scaling \(\alpha=32\), orthogonality weight \(\lambda_{\mathrm{ortho}}=10\), protected subspace size \(k=64\), batch size 8, LoRA dropout 0.05, and the AdamW optimizer. For Llama-3.1-8B-Instruct, we use a learning rate of \(2\times10^{-4}\) and 10 epochs. For Mistral-7B-Instruct-v0.3, we use a learning rate of \(2\times10^{-5}\) and 5 epochs. SVF experts are trained for 20 RL iterations with the default hyperparameters of the original implementation.
\\
\bottomrule
\end{tabularx}
\end{table}

\subsection{Q1: Are Skill-Relevant Singular Directions Concentrated in Top Singular Components?}
\label{sec:distributed-validation}

We first test the geometric premise of PALoRA. If skill-relevant information were concentrated in the largest singular directions, then raw top-\(k\) protection should be sufficient. PALoRA instead predicts that the critical directions are identified by SVF and need not align with the largest singular values.

Figure~\ref{fig:distributed_evidence} visualizes the SVF scaling pattern for the GSM8K expert on Llama-3.1-8B-Instruct, and Table~\ref{tab:subspace_validation} compares SVF-guided selection with raw top-\(k\) and random-\(k\). 
The results support the distributed-skill view. The SVF-selected components are spread across the spectrum and have minimal overlap with the raw top-\(k\) subspace. This difference is behaviorally significant: SVF-guided protection yields the best GSM8K accuracy while keeping recall competitive.

\begin{figure}[t]
    \centering
    \begin{minipage}{0.47\textwidth}
        \centering
        \includegraphics[width=\textwidth]{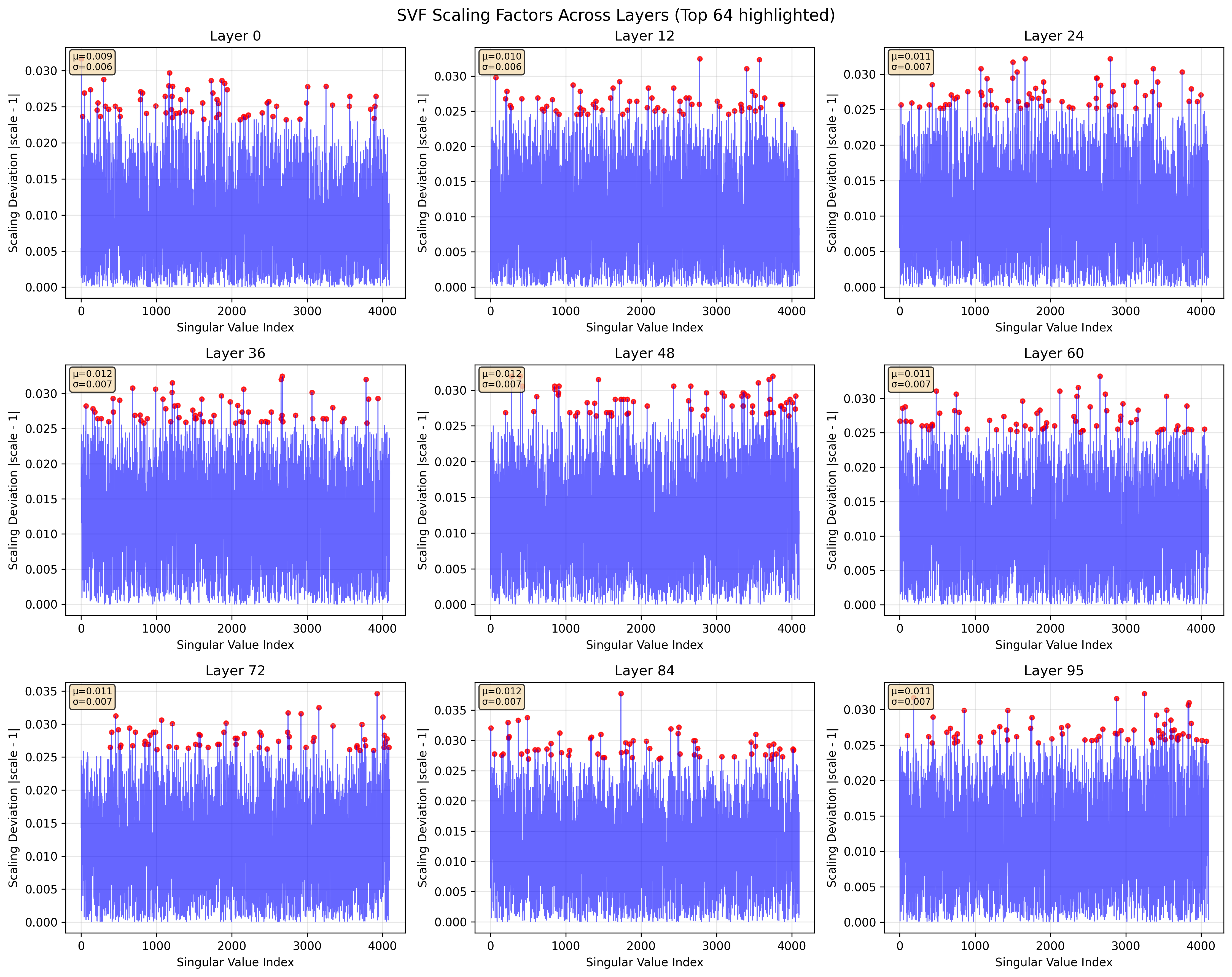}
        \small (a) SVF scaling magnitudes \( |z_i^*-1| \) over singular-value index.
    \end{minipage}
    \hfill
    \begin{minipage}{0.47\textwidth}
        \centering
        \includegraphics[width=\textwidth]{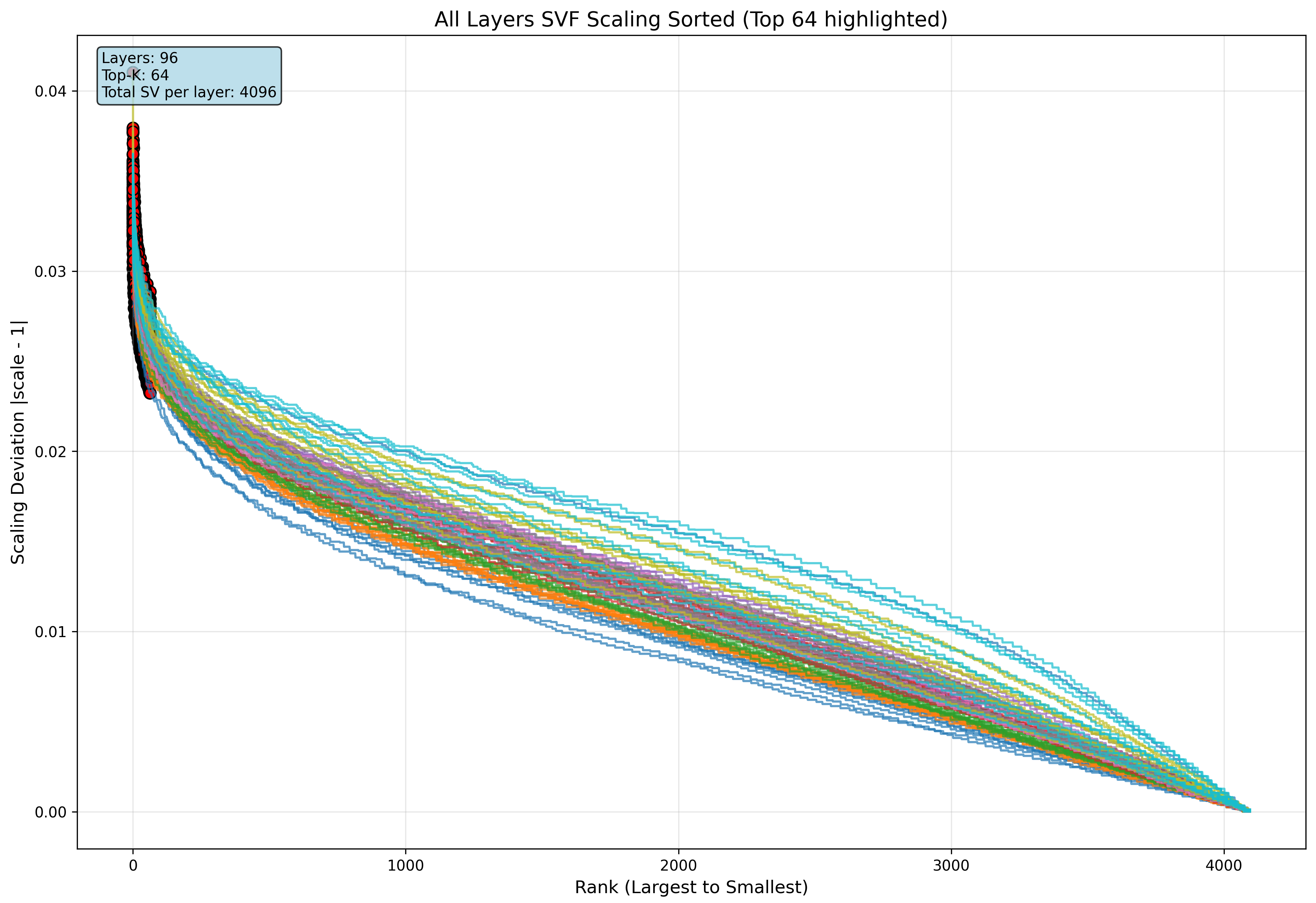}
        \small (b) Top-64 SVF scaling magnitudes \( |z_i^*-1| \), sorted from largest to smallest.
    \end{minipage}
    \caption{Visual evidence for distributed skill information. The SVF scaling factors show a sharp elbow and long tail, while the top SVF-selected components are scattered across singular-value rank rather than aligned with the largest singular values.}
    \label{fig:distributed_evidence}
\end{figure}

\begin{table}[t]
\centering
\small
\caption{Validation of SVF-guided subspace selection on Llama-3.1-8B with the GSM8K expert and 1000 unknown facts.}
\label{tab:subspace_validation}
\setlength{\tabcolsep}{4pt}
\renewcommand{\arraystretch}{1.02}
\begin{tabular}{lcccc}
\toprule
& \multicolumn{2}{c}{\textbf{Performance}} & \multicolumn{2}{c}{\textbf{Overlap with SVF-Guided}} \\
\cmidrule(lr){2-3} \cmidrule(lr){4-5}
\textbf{Method} & \textbf{Recall (\%)} & \textbf{GSM8K (\%)} & \textbf{Micro Overlap} & \textbf{Micro Jaccard} \\
\midrule
SVF-Guided (ours) & 60.50 & \textbf{80.06} & --- & --- \\
Top-\(k\) Raw     & 57.40 & 77.56 & 0.0229 & 0.0116 \\
Random-\(k\)      & \textbf{61.30} & 72.63 & 0.0145 & 0.0073 \\
\bottomrule
\end{tabular}
\end{table}

\subsection{Q2--Q5: Skill Preservation, Task Generalization, Baseline Comparison, and Cross-Architecture Transfer}
\label{sec:main-results}

We next evaluate whether PALoRA preserves reasoning abilities during factual adaptation, whether this behavior extends across tasks, how it compares with existing baselines, and whether these trends remain consistent across model architectures.

We first evaluate PALoRA on Llama-3.1-8B-Instruct. The base model achieves \(78.54\%\) on GSM8K, \(70.7\%\) on MBPP, and \(90.32\%\) on AI2-ARC; after Phase~1, the corresponding SVF experts reach \(83.85\%\), \(72.7\%\), and \(90.11\%\), respectively.

PALoRA consistently preserves reasoning performance better than unconstrained LoRA and prior spectral PEFT baselines, particularly on GSM8K and MBPP. On AI2-ARC, PALoRA remains competitive while avoiding the larger performance degradation observed with LoRA and CorDA at higher factual loads. Overall, PALoRA achieves a strong balance between reasoning preservation and factual recall.

\begin{table*}[t]
\centering
\footnotesize
\renewcommand{\arraystretch}{0.9}
\setlength{\tabcolsep}{3pt}
\caption{Llama-3.1-8B: knowledge recall and skill accuracy across three reasoning benchmarks.}
\label{tab:llama_main}
\begin{tabular}{@{} p{2.2cm} c c c c c c c c c @{}} 
\toprule
 & & \multicolumn{2}{c}{\textbf{LoRA}} & \multicolumn{2}{c}{\textbf{OPLoRA}} & \multicolumn{2}{c}{\textbf{CorDA}} & \multicolumn{2}{c}{\textbf{PALoRA (Ours)}} \\
\cmidrule(lr){3-4} \cmidrule(lr){5-6} \cmidrule(lr){7-8} \cmidrule(lr){9-10}
\textbf{Skill Dataset} & \textbf{\# Facts} & \textbf{Rec.} & \textbf{Skill} & \textbf{Rec.} & \textbf{Skill} & \textbf{Rec.} & \textbf{Skill} & \textbf{Rec.} & \textbf{Skill} \\
\midrule
\multirow{4}{*}{\textbf{GSM8K (Math)}}
& 100  & 61.0 & 74.45 & 64.0 & 75.82 & \textbf{78.0} & 76.65 & 64.0 & \textbf{79.38} \\
& 250  & 63.2 & 70.13 & 64.0 & 75.06 & \textbf{68.0} & 74.98 & 64.0 & \textbf{78.09} \\
& 500  & 62.2 & 72.71 & 66.2 & 77.56 & \textbf{68.8} & 66.94 & 61.6 & \textbf{82.41} \\
& 1000 & 61.7 & 72.33 & 62.3 & 76.65 & 62.2 & 60.5 & \textbf{62.6} & \textbf{81.12} \\
\midrule
\multirow{4}{*}{\textbf{MBPP (Coding)}}
& 100  & 61.0 & 66.70 & 64.0 & 66.67 & \textbf{68.0} & 50.5 & 59.0 & \textbf{67.70} \\
& 250  & 63.2 & 66.70 & 64.0 & 65.66 & \textbf{66.8} & 32.32 & 60.8 & \textbf{68.70} \\
& 500  & 62.2 & 65.70 & \textbf{66.2} & 66.67 & 38.2 & 27.27 & 58.2 & \textbf{67.70} \\
& 1000 & \textbf{62.6} & 64.60 & 62.3 & 67.68 & 56.3 & 12.12 & 61.7 & \textbf{68.70} \\
\midrule
\multirow{4}{*}{\textbf{AI2-ARC (Scientific)}}
& 100  & 61.0 & 86.03 & \textbf{64.0} & \textbf{88.13} & 63.0 & 83.29 & 60.0 & 87.96 \\
& 250  & 63.2 & 81.14 & 64.0 & 87.54 & \textbf{71.2} & 81.63 & 58.4 & \textbf{88.30} \\
& 500  & 62.2 & 87.37 & \textbf{66.2} & 88.05 & 51.8 & 60.31 & 60.2 & \textbf{89.10} \\
& 1000 & 61.7 & 87.67 & 62.3 & \textbf{88.80} & \textbf{68.2} & 50.96 & 59.3 & 88.43 \\
\bottomrule
\end{tabular}
\end{table*}

We next evaluate the same pipeline on Mistral-7B-Instruct-v0.3. The base model achieves GSM8K \(44.20\%\), MBPP \(49.5\%\), and AI2-ARC \(81.65\%\); Phase~1 SVF experts raise these to \(48.29\%\), \(51.5\%\), and \(85.06\%\), respectively.

\begin{table*}[t]
\centering
\footnotesize
\renewcommand{\arraystretch}{0.9}
\setlength{\tabcolsep}{3pt}
\caption{Mistral-7B: knowledge recall and skill accuracy across three reasoning benchmarks.}
\label{tab:mistral_main}
\begin{tabular}{@{} p{2.2cm} c c c c c c c c c @{}} 
\toprule
 & & \multicolumn{2}{c}{\textbf{LoRA}} & \multicolumn{2}{c}{\textbf{OPLoRA}} & \multicolumn{2}{c}{\textbf{CorDA}} & \multicolumn{2}{c}{\textbf{PALoRA (Ours)}} \\
\cmidrule(lr){3-4} \cmidrule(lr){5-6} \cmidrule(lr){7-8} \cmidrule(lr){9-10}
\textbf{Skill Dataset} & \textbf{\# Facts} & \textbf{Rec.} & \textbf{Skill} & \textbf{Rec.} & \textbf{Skill} & \textbf{Rec.} & \textbf{Skill} & \textbf{Rec.} & \textbf{Skill} \\
\midrule
\multirow{4}{*}{\textbf{GSM8K (Math)}}
& 100  & 63.0 & 39.12 & 60.0 & 42.99 & \textbf{78.0} & 41.24 & 63.0 & \textbf{46.40} \\
& 250  & 66.8 & 41.40 & 56.8 & 42.38 & \textbf{69.6} & 35.37 & 65.6 & \textbf{46.17} \\
& 500  & 63.0 & 41.70 & 52.0 & 42.00 & \textbf{66.4} & 25.02 & 65.6 & \textbf{46.93} \\
& 1000 & 63.9 & 39.27 & 52.7 & 40.11 & \textbf{65.2} & 23.96 & 64.5 & \textbf{44.58} \\
\midrule
\multirow{4}{*}{\textbf{MBPP (Coding)}}
& 100  & 63.0 & 48.48 & 58.0 & 48.48 & \textbf{80.0} & \textbf{49.49} & 67.0 & 47.47 \\
& 250  & 66.8 & \textbf{48.48} & 58.4 & 47.47 & \textbf{74.4} & 45.45 & 64.0 & \textbf{48.48} \\
& 500  & 63.0 & 44.44 & 52.0 & 44.44 & \textbf{72.6} & 42.42 & 63.6 & \textbf{46.46} \\
& 1000 & 63.9 & \textbf{51.52} & 52.7 & 47.47 & \textbf{77.9} & 42.42 & 62.2 & 49.49 \\
\midrule
\multirow{4}{*}{\textbf{AI2-ARC (Scientific)}}
& 100  & 63.0 & 80.09 & 58.0 & 81.52 & \textbf{79.0} & 80.30 & 60.0 & \textbf{84.89} \\
& 250  & 66.8 & 79.08 & 58.4 & 79.46 & \textbf{71.2} & 80.00 & 56.8 & \textbf{85.19} \\
& 500  & 63.0 & 77.90 & 52.0 & 78.54 & \textbf{70.4} & 77.53 & 60.4 & \textbf{84.47} \\
& 1000 & 63.9 & 77.82 & 52.7 & 79.92 & \textbf{70.5} & 78.96 & 60.3 & \textbf{84.01} \\
\bottomrule
\end{tabular}
\end{table*}

The trade-off is more pronounced on Mistral than on Llama. LoRA preserves recall but strongly degrades reasoning, and OPLoRA provides only limited protection. PALoRA improves skill retention substantially on GSM8K and AI2-ARC while keeping recall competitive, showing that the same overall pattern transfers across architectures.

\subsection{Ablation Studies}
\label{sec:ablation}

We next conduct ablation studies to analyze the sensitivity of PALoRA to its primary hyperparameters on Llama-3.1-8B-Instruct using the GSM8K expert and 1000 injected facts. The results reveal a consistent plasticity–stability trade-off. Increasing the LoRA rank improves factual recall by expanding adaptation capacity, but also increases interference with the protected reasoning subspace, leading to weaker GSM8K retention. In contrast, stronger orthogonality regularization improves reasoning preservation at the cost of moderately reduced factual acquisition. We further observe that PALoRA remains relatively robust across a broad range of protected subspace sizes \(k\), with \(k=64\) providing the best overall balance between recall and skill preservation in our experiments.

\begin{table*}[t]
\centering
\small
\caption{Ablations on Llama-3.1-8B with the GSM8K expert and 1000 unknown facts.}
\label{tab:ablation_all}
\setlength{\tabcolsep}{6pt}
\begin{tabular}{ccc|ccc|ccc}
\toprule
\multicolumn{3}{c|}{\textbf{LoRA rank \(r\)}} &
\multicolumn{3}{c|}{\textbf{Orthogonality \(\lambda_{\text{ortho}}\)}} &
\multicolumn{3}{c}{\textbf{Protected subspace \(k\)}} \\
\midrule
Value & Recall & GSM8K & Value & Recall & GSM8K & Value & Recall & GSM8K \\
\midrule
8   & 60.5 & 81.0 & 0.1  & 63.8 & 76.6 & 16  & 60.9 & 79.8 \\
16  & 59.8 & 80.1 & 1    & 62.6 & 81.1 & 32  & 62.0 & 77.9 \\
32  & 65.7 & 72.3 & 10   & 60.1 & 80.1 & 64  & 60.0 & 80.1 \\
64  & 70.5 & 69.4 & 100  & 58.6 & 82.9 & 128 & 61.5 & 78.2 \\
\bottomrule
\end{tabular}
\end{table*}

\subsection{Training Dynamics and Computational Overhead}
\label{sec:dynamics}

We finally analyze the optimization dynamics and computational overhead of PALoRA. Figure~\ref{fig:dynamics} shows that PALoRA successfully acquires new factual knowledge while maintaining minimal interference with the protected reasoning subspace throughout training. Compared with standard LoRA, PALoRA introduces two additional components: an offline SVF specialization phase and a lightweight orthogonality regularizer during adaptation. Despite these additions, the computational overhead remains modest. In our experiments, PALoRA increases training runtime by only approximately \(5\)–\(10\%\) relative to standard LoRA fine-tuning, while the memory overhead associated with storing the protected singular vectors remains below \(1\%\) of the total model footprint.

\begin{figure}[!htbp]
    \centering
    \begin{minipage}{0.47\textwidth}
        \centering
        \includegraphics[width=\textwidth]{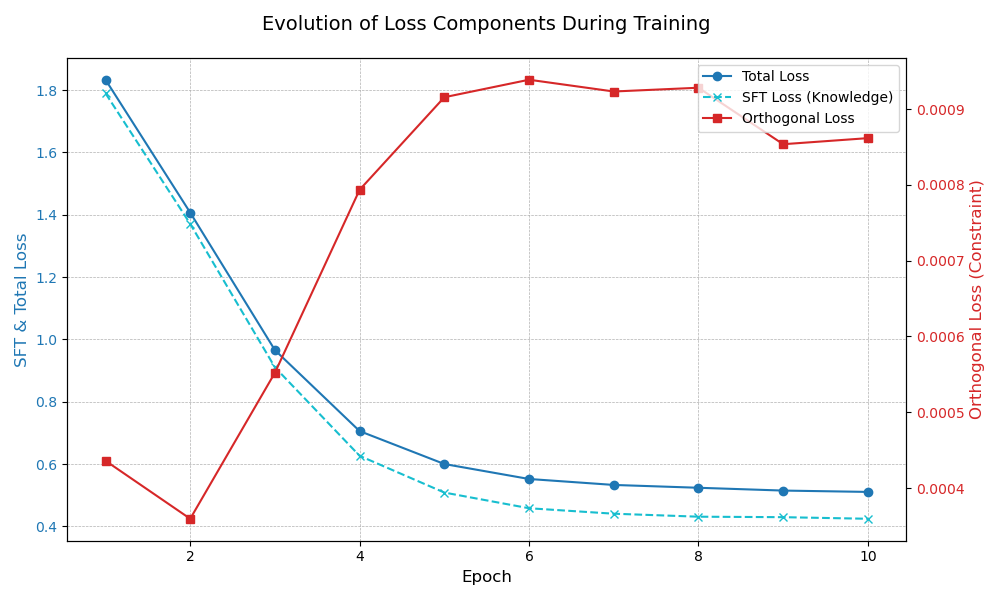}
        \small (a) Loss evolution during Phase~2.
    \end{minipage}
    \hfill
    \begin{minipage}{0.47\textwidth}
        \centering
        \includegraphics[width=\textwidth]{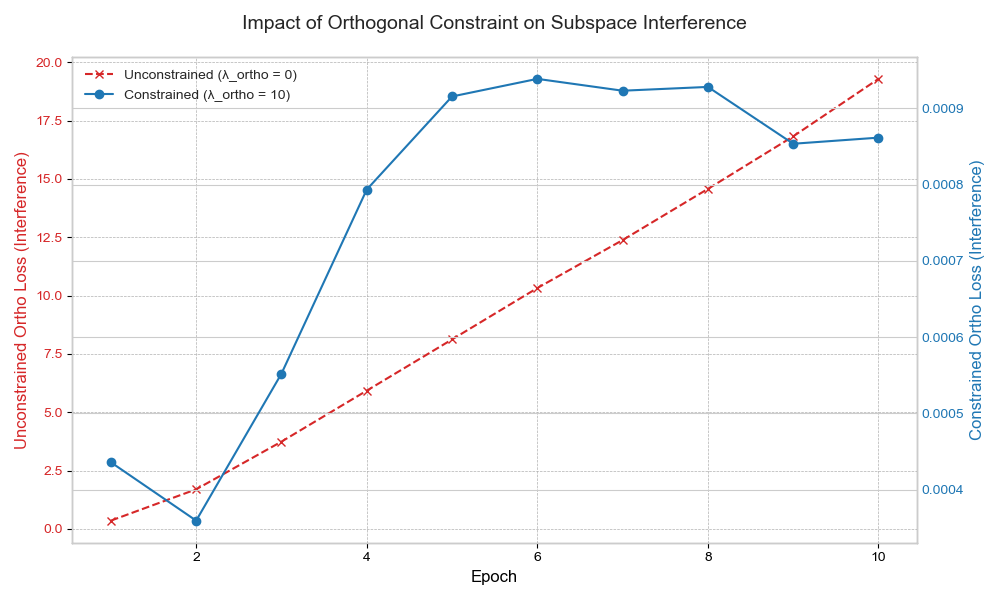}
        \small (b) Interference with and without the constraint.
    \end{minipage}
    \caption{Training dynamics for Llama-3.1-8B with the GSM8K expert and 1000 unknown facts. The SFT loss decreases steadily, while the constraint suppresses projection of the LoRA update onto the critical skill subspace by more than an order of magnitude.}
    \label{fig:dynamics}
\end{figure}

\section{Conclusion}
\label{secconclusion}

We investigated the plasticity–stability trade-off that arises when adapting large language models to new factual knowledge while preserving existing reasoning abilities. Our analysis shows that reasoning-relevant information in Transformer MLP weights is not concentrated solely in the dominant singular directions, but is instead distributed across the singular spectrum. This finding challenges the common assumption that protecting only the top spectral components is sufficient for preserving downstream skills during adaptation.

Building on this insight, we introduced PALoRA, a two-stage framework for skill-preserving knowledge injection. PALoRA first trains a Singular Value Finetuning expert to identify a skill-critical subspace, then constrains LoRA-based factual updates to remain structurally orthogonal to that subspace. We further provided a theoretical motivation for why top-\(k\)-only spectral protection can fail, and empirically demonstrated the effectiveness of PALoRA across two model families and three reasoning domains. Our results show that PALoRA consistently preserves reasoning performance better than prior spectral PEFT baselines while maintaining competitive factual recall.

A limitation of PALoRA is that its knowledge-injection capacity depends on the strength of the subspace protection: more aggressive preservation of skill-critical directions can reduce factual adaptation performance. This reflects an inherent trade-off between retaining prior capabilities and incorporating new knowledge. Future work should extend PALoRA to simultaneous preservation of multiple skills, broader model architectures, and continual adaptation settings in which models must integrate evolving knowledge over time without cumulative degradation.

\bibliographystyle{unsrtnat}
\bibliography{references}

\newpage
%%%%%%%%%%%%%%%%%%%%%%%%%%%%%%%%%%%%%%%%%%%%%%%%%%%%%%%%%%%%
\appendix

\section*{Appendix}
This appendix provides supplementary material supporting the theoretical, methodological, and empirical results of the paper. We first present the full proofs of the main theorem and corollaries, then provide a consolidated notation reference covering the symbols used throughout the paper. We next report additional compute, parameter, and memory analyses, followed by detailed descriptions of the datasets, training configuration, and evaluation protocol. We then describe the baseline adaptation and fairness protocol used in our comparisons, and finally provide additional SVF adapter visualizations to further illustrate the distributed structure of skill-relevant singular components.

\section{Full Proof of Theorem~\ref{thm:distributed} and Corollaries}
\label{app:full-proof}

\subsection{Notation}

Let \(\mathbf{W} \in \mathbb{R}^{m \times n}\) be an MLP weight matrix of a Transformer layer.
Its singular value decomposition is
\[
\mathbf{W}
=
\mathbf{U}\boldsymbol{\Sigma}\mathbf{V}^\top
=
\sum_{i=1}^{r} \sigma_i \mathbf{u}_i \mathbf{v}_i^\top,
\]
where \(r=\operatorname{rank}(\mathbf{W})\),
\(\mathbf{u}_i \in \mathbb{R}^{m}\) and \(\mathbf{v}_i \in \mathbb{R}^{n}\)
are the left and right singular vectors, and
\(\sigma_1 \ge \sigma_2 \ge \cdots \ge \sigma_r > 0\).

Let \(\mathbf{z} \in \mathbb{R}^{r}\) denote the SVF scaling vector.
The scaled weight matrix is
\[
\mathbf{W}_{\mathbf{z}}
=
\mathbf{U}\operatorname{diag}(\mathbf{z}\odot\boldsymbol{\sigma})\mathbf{V}^\top
=
\sum_{i=1}^{r} z_i \sigma_i \mathbf{u}_i \mathbf{v}_i^\top.
\]
Let \(\mathbf{z}^*\) denote the optimal SVF scaling vector obtained after skill fine-tuning on task \(T\).
Let \(L_T(\mathbf{W})\) denote the task-specific loss over dataset \(\mathcal{D}_T\).

Let \(\mathcal{U}_T \subset \mathbb{R}^{n}\) be the task-specific input subspace associated with skill \(T\),
with orthogonal projector
\[
\mathbf{P}_T = \mathbf{V}_T \mathbf{V}_T^\top,
\qquad
\mathbf{V}_T \in \mathbb{R}^{n \times s},
\qquad
s \ll n.
\]
We define the set of skill-relevant singular components by
\[
\mathcal{S}_T
=
\left\{
i \in \{1,\dots,r\}
\;\middle|\;
|z_i^* - 1| > \varepsilon
\right\},
\]
for a threshold \(\varepsilon > 0\).

We use the Frobenius inner product
\[
\langle \mathbf{A}, \mathbf{B} \rangle_F = \operatorname{tr}(\mathbf{A}^\top \mathbf{B}),
\]
and the Frobenius norm \(\|\mathbf{A}\|_F\).

\subsection{Formal Assumptions}
\label{app:assump}

\begin{assumption}[Pre-training data diversity]
\label{ass:div}
The pre-training corpus is drawn from a distribution \(\mathcal{D}_{\mathrm{prior}}\) with broad spectral support.
The input covariance matrix
\[
\mathbf{C}_{\mathrm{prior}}
=
\mathbb{E}_{\mathbf{x}\sim\mathcal{D}_{\mathrm{prior}}}
[\mathbf{x}\mathbf{x}^\top]
\in \mathbb{R}^{n\times n}
\]
has full rank, and its eigenvalues satisfy
\[
\lambda_1 \ge \lambda_2 \ge \cdots \ge \lambda_n > 0,
\qquad
\lambda_i \asymp i^{-\alpha},
\qquad
\alpha > 0.
\]
Consequently, the pre-trained weight matrix \(\mathbf{W}_0\) concentrates its principal
singular vectors along directions of maximum input variance under \(\mathcal{D}_{\mathrm{prior}}\).
\end{assumption}

\begin{assumption}[Skill compositionality and misalignment]
\label{ass:skill}
A skill \(T\) corresponds to a compositional function
\(f_T : \mathbb{R}^{n} \to \mathbb{R}^{d}\)
that depends only on a low-dimensional, task-specific subspace
\(\mathcal{U}_T \subset \mathbb{R}^{n}\).
Formally, there exists a projector
\(\mathbf{P}_T = \mathbf{V}_T \mathbf{V}_T^\top\),
\(\mathbf{V}_T \in \mathbb{R}^{n \times s}\),
such that
\[
f_T(\mathbf{x}) = f_T(\mathbf{P}_T \mathbf{x})
\quad \text{almost surely under the task distribution.}
\]
Moreover, the task subspace is misaligned with the top-\(k\) right singular directions of \(\mathbf{W}_0\):
for every \(i \le k\),
\begin{equation}
    \|\mathbf{P}_T \mathbf{v}_i\|_2^2 \le \delta,
    \qquad
    0 < \delta \ll 1.
    \label{eq:misalignment_assumption}
\end{equation}
\end{assumption}

\begin{assumption}[Gradient structure]
\label{ass:grad}
Skill adaptation is performed via first-order gradient descent on \(L_T(\mathbf{W})\),
starting from the pre-trained weight matrix \(\mathbf{W}_0\).
The gradient at \(\mathbf{W}_0\) decomposes as
\begin{equation}
    \nabla_{\mathbf{W}} L_T(\mathbf{W}_0)
    =
    \mathbf{G}_T + \mathbf{N},
    \label{eq:grad_decomp}
\end{equation}
where \(\mathbf{G}_T\) is a low-rank structured term with
\(\operatorname{rank}(\mathbf{G}_T) \le s\),
and its row space lies in \(\mathcal{U}_T\).
Accordingly,
\[
\mathbf{G}_T
=
\sum_{j=1}^{s} \mathbf{g}_j \mathbf{h}_j^\top,
\qquad
\mathbf{h}_j \in \mathcal{U}_T \ \forall j,
\]
and \(\mathbf{g}_j \in \mathbb{R}^{m}\).
The residual term \(\mathbf{N}\) is a zero-mean noise matrix with small Frobenius norm
\[
\|\mathbf{N}\|_F \le \epsilon_{\mathrm{noise}}.
\]
\end{assumption}

\subsection{Proof of Theorem~\ref{thm:distributed}}

\paragraph{Detailed version of Theorem~\ref{thm:distributed} (Distributed encoding of skill-relevant components).}
Under Assumptions~\ref{ass:div}, \ref{ass:skill}, and \ref{ass:grad},
for any \(k < \operatorname{rank}(\mathbf{W})\) satisfying the misalignment condition in
Assumption~\ref{ass:skill}, the set of skill-relevant components \(\mathcal{S}_T\) is
\emph{not} contained in the top-\(k\) index set \(\{1,2,\dots,k\}\).

\paragraph{Proof.}
We follow the same four-step structure as the handwritten manuscript.

\paragraph{Step 1: Gradient spectrum decomposition.}
Decompose the task gradient \(\mathbf{G}_T\) in the singular basis of \(\mathbf{W}\).
Define the spectral loading coefficients
\begin{equation}
    \gamma_i
    =
    \langle \mathbf{u}_i \mathbf{v}_i^\top, \mathbf{G}_T \rangle_F
    =
    \mathbf{u}_i^\top \mathbf{G}_T \mathbf{v}_i,
    \qquad i=1,\dots,r.
    \label{eq:gamma_def}
\end{equation}

Substituting
\(
\mathbf{G}_T = \sum_{j=1}^{s} \mathbf{g}_j \mathbf{h}_j^\top
\),
we obtain
\begin{equation}
    \gamma_i
    =
    \sum_{j=1}^{s}
    (\mathbf{u}_i^\top \mathbf{g}_j)
    (\mathbf{h}_j^\top \mathbf{v}_i).
    \label{eq:gamma_expand}
\end{equation}

The factor \(\mathbf{h}_j^\top \mathbf{v}_i = \langle \mathbf{h}_j, \mathbf{v}_i \rangle\)
is the projection of \(\mathbf{h}_j \in \mathcal{U}_T\) onto \(\mathbf{v}_i\).
Since \(\mathbf{h}_j \in \mathcal{U}_T\), we have
\[
|\mathbf{h}_j^\top \mathbf{v}_i|
=
|\mathbf{h}_j^\top \mathbf{P}_T \mathbf{v}_i|
\le
\|\mathbf{h}_j\|_2 \, \|\mathbf{P}_T \mathbf{v}_i\|_2.
\]
By Assumption~\ref{ass:skill}, for every \(i \le k\),
\[
\|\mathbf{P}_T \mathbf{v}_i\|_2 \le \sqrt{\delta},
\]
hence
\[
|\mathbf{h}_j^\top \mathbf{v}_i|
\le
\sqrt{\delta}\,\|\mathbf{h}_j\|_2.
\]
Applying Cauchy--Schwarz in Eq.~\ref{eq:gamma_expand} yields
\begin{equation}
    |\gamma_i|
    \le
    C_1 \sqrt{\delta},
    \qquad
    \forall i \le k,
    \label{eq:gamma_small}
\end{equation}
for some constant \(C_1>0\) depending on \(\mathbf{G}_T\).

Conversely, since \(\mathbf{G}_T \neq \mathbf{0}\) and
\(\mathcal{U}_T\) is generically misaligned with
\(\operatorname{span}\{\mathbf{v}_1,\dots,\mathbf{v}_k\}\),
a dimension-counting argument on the Grassmannian implies the existence of at least one
index \(i^*>k\) such that
\begin{equation}
    |\gamma_{i^*}|
    \ge
    \frac{c_T}{\sqrt{n-k}},
    \qquad
    c_T = \frac{\|\mathbf{G}_T\|_F}{\sqrt{s}} > 0.
    \label{eq:gamma_large}
\end{equation}

\paragraph{Step 2: Singular-value perturbation under a full-parameter gradient update.}
Consider one gradient descent step on the full weight matrix:
\begin{equation}
    \mathbf{W}_1
    =
    \mathbf{W}_0 - \eta \nabla_{\mathbf{W}}L_T(\mathbf{W}_0)
    =
    \mathbf{W}_0 - \eta(\mathbf{G}_T + \mathbf{N}).
    \label{eq:gd_step}
\end{equation}
By first-order perturbation theory for singular values~\cite{horn2012matrix,stewart1990matrix,golub2013matrix},
the change in the \(i\)-th singular value satisfies
\begin{equation}
    \frac{|\Delta \sigma_i|}{\sigma_i}
    =
    \eta |\gamma_i|
    +
    O(\eta \|\mathbf{N}\|_F + \eta^2),
    \label{eq:relative_sigma_change}
\end{equation}
where \(\Delta \sigma_i = \sigma_i(\mathbf{W}_1)-\sigma_i(\mathbf{W}_0)\).  (This relation also follows from Weyl's inequality~\cite{weyl1912asymptotische} in the perturbation regime.)

Combining Eq.~\ref{eq:relative_sigma_change} with
Eqs.~\ref{eq:gamma_small}--\ref{eq:gamma_large}, we obtain
\[
\frac{|\Delta \sigma_i|}{\sigma_i}
=
O(\eta \sqrt{\delta}),
\qquad i \le k,
\]
whereas for the index \(i^*>k\),
\[
\frac{|\Delta \sigma_{i^*}|}{\sigma_{i^*}}
\ge
\eta \frac{c_T}{\sqrt{n-k}}
-
O(\eta \|\mathbf{N}\|_F + \eta^2).
\]
Thus, under a direct gradient update, the top-\(k\) singular values change only weakly,
whereas the off-principal component \(i^*\) changes at non-negligible scale.

\paragraph{Step 3: SVF scaling vector and skill relevance.}
In the SVF parameterization,
\[
\mathbf{W}_{\mathbf{z}}
=
\sum_{i=1}^{r} z_i \sigma_i \mathbf{u}_i \mathbf{v}_i^\top.
\]
The gradient of the task loss with respect to \(z_i\) is
\begin{equation}
    \frac{\partial L_T}{\partial z_i}
    =
    \left\langle
    \nabla_{\mathbf{W}}L_T,\,
    \sigma_i \mathbf{u}_i \mathbf{v}_i^\top
    \right\rangle_F
    =
    \sigma_i \gamma_i
    +
    O(\sigma_i \|\mathbf{N}\|_F).
    \label{eq:svf_grad}
\end{equation}

Starting from the identity initialization \(z_i=1\), one gradient step in the SVF
parameterization gives
\begin{equation}
    z_i^{(1)} - 1
    =
    -\eta_z \sigma_i \gamma_i
    +
    O(\eta_z \sigma_i \|\mathbf{N}\|_F + \eta_z^2).
    \label{eq:zi_step}
\end{equation}
Therefore, to first order,
\begin{equation}
    |z_i^* - 1|
    =
    \eta_z \sigma_i |\gamma_i|
    +
    O(\eta_z \sigma_i \|\mathbf{N}\|_F + \eta_z^2).
    \label{eq:zi_importance}
\end{equation}

For \(i \le k\), Eq.~\ref{eq:gamma_small} yields
\[
|z_i^* - 1|
\le
C \eta_z \sigma_i \sqrt{\delta}
+
O(\eta_z \sigma_i \|\mathbf{N}\|_F + \eta_z^2),
\]
whereas for the index \(i^*>k\), Eq.~\ref{eq:gamma_large} yields
\[
|z_{i^*}^* - 1|
\ge
\eta_z \sigma_{i^*}\frac{c_T}{\sqrt{n-k}}
-
O(\eta_z \sigma_{i^*}\|\mathbf{N}\|_F + \eta_z^2).
\]

\paragraph{Step 4: Connection and integration.}
Choose a threshold \(\varepsilon\) of the same order as the top-\(k\) deviation scale,
for example
\[
\varepsilon = C \eta_z \sigma_{\max}\sqrt{\delta},
\]
up to the higher-order residual terms.
Then every index \(i \le k\) satisfies
\[
|z_i^* - 1| \le \varepsilon,
\]
whereas the index \(i^*>k\) satisfies
\[
|z_{i^*}^* - 1| > \varepsilon
\]
whenever the lower bound above dominates the residual terms.
Hence \(i^* \in \mathcal{S}_T\) with \(i^*>k\), and therefore
\[
\mathcal{S}_T \cap \{k+1,\dots,r\} \neq \emptyset.
\]
This proves that the set of skill-relevant singular components is not contained in
the top-\(k\) index set.
\hfill\(\square\)

\subsection{Corollaries}

\paragraph{Detailed version of Corollary~\ref{cor:topk-fail} (Sub-optimality of top-\(k\) projection).}
Under Assumptions~\ref{ass:div}--\ref{ass:grad},
any compression or fine-tuning method that preserves or updates only the top-\(k\)
singular components fails to preserve all skill-relevant directions.
Under the corresponding local approximation, the resulting degradation satisfies
\[
\mathbb{E}\!\left[L_T(\mathbf{W}_{\text{top-}k}) - L_T(\mathbf{W}^*)\right]
\ge
\Omega\!\left(\frac{c_T^2}{n-k}\right),
\]
where \(\mathbf{W}_{\text{top-}k}\) denotes the matrix obtained by retaining only the
top-\(k\) protected components, and \(\mathbf{W}^*\) denotes the unconstrained
skill-adapted optimum.

\paragraph{Detailed version of Corollary~\ref{cor:svf-opt} (SVF as a diagnostic probe).}
The SVF scaling vector \(\mathbf{z}^*\) correctly identifies skill-relevant components
despite their singular-value rank, because
\[
|z_i^* - 1| \propto \sigma_i |\gamma_i|
\]
to first order.
Therefore, a component with a relatively small raw singular value can still exhibit a
large deviation from \(1\) when its task-gradient projection is sufficiently strong.

\paragraph{Detailed version of Corollary~\ref{cor:multiskill} (Multi-skill distribution).}
Suppose the model is fine-tuned on \(M\) skills \(\{T_1,\dots,T_M\}\), each with its own
task subspace \(\mathcal{U}_{T_m}\) satisfying Assumption~\ref{ass:skill}.
Let \(\mathcal{S}_{T_m}\) denote the corresponding skill-relevant index set, and define
\[
\mathcal{S}
=
\bigcup_{m=1}^{M}\mathcal{S}_{T_m}.
\]
If these task subspaces are generically incoherent, then \(\mathcal{S}\) spreads across
a broad portion of the singular spectrum.
Consequently, no fixed top-\(k\) subset can capture all skill-relevant components for all skills simultaneously.

\subsection{Connection to Prior Work}
\label{app:connection}

Our theoretical framework provides a unifying geometric perspective on several recent PEFT methods.
\textbf{MiLoRA}~\cite{wang2025milora} fine‑tunes only the minor singular components, hypothesising that principal components store task‑agnostic knowledge – our analysis formalises \emph{why} the minor components carry skill‑specific adaptations.
\textbf{SVFT}~\cite{lingam2024svft} learns sparse combinations of singular vectors, with the insight that fine‑grained control over expressivity comes from choosing \emph{which} singular components to update, a direct consequence of Theorem~\ref{thm:distributed}.
\textbf{PiSSA}~\cite{pissa2024} and \textbf{CorDA}~\cite{yang2024corda} optimise the principal or context‑oriented components; while suitable for general fine‑tuning, Corollary~\ref{cor:topk-fail} proves that protecting only principal components is strictly sub‑optimal for decoupled adaptation, as specialised skills are fundamentally distributed.

\subsection{Geometric Interpretation and Limitation}
\label{app:discussion}

The proof formalizes a simple geometric picture.
Pre-training aligns the principal singular directions of the weight matrix with directions
of high variance in the corpus, and these directions are largely task-agnostic.
A reasoning skill, by contrast, activates a task-specific subspace that can be substantially
misaligned with those principal directions.
As a result, skill-relevant information may live in singular components that are not among
the largest singular values.

The main limitation is the misalignment assumption in Assumption~\ref{ass:skill}.
If a skill happens to align strongly with dominant corpus directions, then the top singular
vectors may already capture a large portion of the skill, and the distributed-spectrum
conclusion may weaken.
The theorem should therefore be interpreted as a structured explanation for the regime
observed in complex reasoning tasks, rather than as a universal statement for every task.

\section{Global Notation Reference}
\label{app:notation}

This section consolidates the notation used throughout the paper.
For readability, we group symbols by their role in the overall training pipeline~\ref{tab:notation_general}, the PALoRA parameterization~\ref{tab:notation_method}, and the theoretical analysis~\ref{tab:notation_theory}.

\begin{table*}[!htbp]
    \centering
    \small
    \caption{General notation, data, and optimization variables used throughout the paper.}
    \label{tab:notation_general}
    \begin{tabularx}{\textwidth}{>{\raggedright\arraybackslash}p{0.27\textwidth}X}
        \toprule
        \textbf{Symbol} & \textbf{Description} \\
        \midrule
        \(\mathcal{M}(\theta_{\text{base}})\) & Pre-trained base language model with frozen parameters \(\theta_{\text{base}}\). \\
        \(\theta_{\text{base}}\) & Parameters of the frozen base model. \\
        \(\theta_{\text{LoRA}}\) & Trainable LoRA parameters, collected as \(\{\mathbf{B}_l,\mathbf{A}_l\}_{l=1}^{L}\). \\
        \(\mathcal{D}_{\text{skill}}\) & Skill dataset used in Phase~1 to train the SVF expert. \\
        \(\mathcal{D}_{\text{know}}\) & Knowledge dataset used in Phase~2 for factual knowledge injection. \\
        \(\mathcal{D}_T\) & Task-specific dataset associated with skill or task \(T\) in the theoretical analysis. \\
        \(\mathcal{D}_{\mathrm{prior}}\) & Pre-training data distribution used in the formal assumptions. \\
        \(L\) & Number of adapted MLP layers. \\
        \(l\) & Layer index. \\
        \(T\) & Number of Phase~2 training steps in Algorithm~\ref{alg:training}; in the theory section, \(T\) also denotes a task or skill, with the meaning determined by context. \\
        \(T_m\) & The \(m\)-th skill in the multi-skill analysis. \\
        \(k\) & Number of singular directions selected as the protected skill-critical subspace in each layer. \\
        \(r\) & LoRA rank in the adaptation module; when used in the theory section as \(r=\operatorname{rank}(\mathbf{W})\), it denotes the rank of the weight matrix. \\
        \(\alpha\) & LoRA scaling hyperparameter in the experiments; in Assumption~\ref{ass:div}, \(\alpha>0\) also denotes the spectral decay exponent of the covariance eigenvalues. \\
        \(\lambda_{\text{ortho}}\) & Weight of the orthogonality regularization term in Phase~2. \\
        \(\mathcal{L}_{\text{SFT}}\) & Supervised fine-tuning loss used for factual knowledge injection. \\
        \(\mathcal{L}_{\text{ortho},l}\) & Orthogonality penalty applied at layer \(l\). \\
        \(\mathcal{L}_{\text{total}}\) & Total Phase~2 objective, \(\mathcal{L}_{\text{SFT}} + \lambda_{\text{ortho}} \sum_{l=1}^{L}\mathcal{L}_{\text{ortho},l}\). \\
        \bottomrule
    \end{tabularx}
\end{table*}

\begin{table*}[!htbp]
    \centering
    \small
    \caption{Notation for the PALoRA parameterization, SVD quantities, and protected subspaces.}
    \label{tab:notation_method}
    \begin{tabularx}{\textwidth}{>{\raggedright\arraybackslash}p{0.27\textwidth}X}
        \toprule
        \textbf{Symbol} & \textbf{Description} \\
        \midrule
        \(\mathbf{W}_{\text{mlp}}^{(l)}\) & Pre-trained MLP weight matrix at layer \(l\), with shape \(\mathbb{R}^{d_{\text{out}} \times d_{\text{in}}}\). \\
        \(\mathbf{W}_{\text{mlp}}^{\prime(l)}\) & Adapted MLP weight matrix after combining the SVF and LoRA components. \\
        \(\operatorname{SVF}(\mathbf{W})\) & Singular Value Fine-Tuning transformation applied to a pre-trained weight matrix \(\mathbf{W}\). \\
        \(\operatorname{LoRA}(\mathbf{W})\) & Low-rank adaptation term applied to a pre-trained weight matrix \(\mathbf{W}\). \\
        \(\Delta \mathbf{W}_l\) & LoRA update at layer \(l\), defined as \(\mathbf{B}_l\mathbf{A}_l\). \\
        \(\mathbf{B}_l\) & Left LoRA factor for layer \(l\), with shape \(\mathbb{R}^{d_{\text{out}} \times r}\). \\
        \(\mathbf{A}_l\) & Right LoRA factor for layer \(l\), with shape \(\mathbb{R}^{r \times d_{\text{in}}}\). \\
        \(d_{\text{out}}, d_{\text{in}}\) & Output and input dimensions of an adapted MLP weight matrix. \\
        \(\mathbf{W}_{\text{mlp}}^{(l)} = \mathbf{U}_l \operatorname{diag}(\boldsymbol{\sigma}_l)\mathbf{V}_l^\top\) & Singular value decomposition of the pre-trained MLP weight matrix at layer \(l\). \\
        \(\mathbf{U}_l, \mathbf{V}_l\) & Left and right singular-vector matrices for layer \(l\). \\
        \(\boldsymbol{\sigma}_l\) & Vector of singular values of \(\mathbf{W}_{\text{mlp}}^{(l)}\). \\
        \(\sigma_{l,i}\) & \(i\)-th singular value of the pre-trained MLP weight matrix at layer \(l\). \\
        \(\mathbf{u}_{l,i}\) & \(i\)-th left singular vector of layer \(l\), used in the protected output subspace. \\
        \(\mathbf{v}_{l,i}\) & \(i\)-th right singular vector of layer \(l\). \\
        \(\mathbf{z}_l\) & Trainable SVF scaling vector for layer \(l\). \\
        \(\mathbf{z}_l^*\) & Frozen SVF scaling vector obtained after Phase~1 training. \\
        \(\mathbf{1}\) & All-ones vector used as the identity reference for SVF scaling. \\
        \(\mathcal{I}_{\text{crit},l}\) & Set of top-\(k\) singular-component indices in layer \(l\) selected by the magnitude of \(|z_{l,i}^*-1|\). \\
        \bottomrule
    \end{tabularx}
\end{table*}

\begin{table*}[!htbp]
    \centering
    \small
    \caption{Notation used in the theoretical analysis and proofs.}
    \label{tab:notation_theory}
    \begin{tabularx}{\textwidth}{>{\raggedright\arraybackslash}p{0.27\textwidth}X}
        \toprule
        \textbf{Symbol} & \textbf{Description} \\
        \midrule
        \(\mathbf{W} \in \mathbb{R}^{m \times n}\) & Generic pre-trained weight matrix used in the theoretical analysis. \\
        \(\mathbf{W}_0\) & Initial pre-trained weight matrix before task adaptation. \\
        \(\mathbf{W}_1\) & Weight matrix after one gradient-based update step. \\
        \(\mathbf{W}_{\mathbf{z}}\) & SVF-parameterized weight matrix obtained by scaling the singular values of \(\mathbf{W}\). \\
        \(\mathbf{W}^*\) & Unconstrained skill-adapted optimum in the corollary analysis. \\
        \(\mathbf{W}_{\text{top-}k}\) & Matrix obtained by retaining or protecting only the top-\(k\) singular components. \\
        \(m,n\) & Output and input dimensions of the generic theoretical weight matrix \(\mathbf{W}\). \\
        \(\mathbf{U}, \boldsymbol{\Sigma}, \mathbf{V}\) & SVD factors of \(\mathbf{W}\), where \(\boldsymbol{\Sigma}=\operatorname{diag}(\sigma_1,\dots,\sigma_r)\). \\
        \(\sigma_i\) & \(i\)-th singular value of \(\mathbf{W}\). \\
        \(\mathbf{u}_i, \mathbf{v}_i\) & \(i\)-th left and right singular vectors of \(\mathbf{W}\). \\
        \(\mathcal{U}_T\) & Low-dimensional task-specific input subspace associated with skill \(T\). \\
        \(\mathbf{P}_T\) & Orthogonal projector onto \(\mathcal{U}_T\), written as \(\mathbf{V}_T\mathbf{V}_T^\top\). \\
        \(\mathbf{V}_T\) & Basis matrix whose columns span the task-specific subspace \(\mathcal{U}_T\). \\
        \(f_T\) & Task-specific compositional function associated with skill \(T\). \\
        \(\mathbf{C}_{\mathrm{prior}}\) & Input covariance matrix induced by the pre-training distribution. \\
        \(\lambda_i\) & \(i\)-th eigenvalue of the covariance matrix \(\mathbf{C}_{\mathrm{prior}}\). \\
        \(\mathbf{G}_T\) & Low-rank structured component of the task gradient. \\
        \(\mathbf{N}\) & Zero-mean noise component in the gradient decomposition. \\
        \(\mathbf{g}_j, \mathbf{h}_j\) & Factors in the decomposition \(\mathbf{G}_T=\sum_{j=1}^{s}\mathbf{g}_j\mathbf{h}_j^\top\), with \(\mathbf{h}_j \in \mathcal{U}_T\). \\
        \(s\) & Rank bound on the structured task-gradient term \(\mathbf{G}_T\), and dimension parameter of the task subspace basis. \\
        \(\gamma_i\) & Spectral loading coefficient \(\mathbf{u}_i^\top \mathbf{G}_T \mathbf{v}_i\), measuring alignment of the task gradient with singular component \(i\). \\
        \(\Delta \sigma_i\) & Change in the \(i\)-th singular value after a gradient update. \\
        \(\eta\) & Learning rate for the full weight-space gradient step in the perturbation analysis. \\
        \(\eta_z\) & Learning rate for the SVF scaling-vector update. \\
        \(\delta\) & Misalignment parameter controlling how weakly the task subspace aligns with the top-\(k\) right singular directions. \\
        \(\varepsilon\) & Threshold used to define the set of skill-relevant singular components. \\
        \(\epsilon_{\mathrm{noise}}\) & Upper bound on the Frobenius norm of the noise term \(\mathbf{N}\). \\
        \(c_T\) & Task-dependent positive constant, defined from \(\|\mathbf{G}_T\|_F / \sqrt{s}\) in the proof. \\
        \(\mathcal{S}_T\) & Set of skill-relevant singular indices identified by the SVF scaling deviation \(|z_i^*-1|\). \\
        \(\mathcal{S}\) & Union of skill-relevant singular indices across multiple skills. \\
        \(M\) & Number of skills considered in the multi-skill corollary. \\
        \(\langle \mathbf{A}, \mathbf{B} \rangle_F\) & Frobenius inner product between matrices \(\mathbf{A}\) and \(\mathbf{B}\). \\
        \(\|\mathbf{A}\|_F\) & Frobenius norm of matrix \(\mathbf{A}\). \\
        \bottomrule
    \end{tabularx}
\end{table*}

\section{Compute, Parameter Count, and Memory Analysis}
\label{sec:appendix_compute_memory}

This section analyzes the practical overhead of PALoRA from first principles. We distinguish three quantities: \emph{trainable parameter count}, \emph{additional stored state}, and \emph{computational overhead}. The discussion is analytical and follows directly from the PALoRA formulation in Sections~\ref{sec:method} and \ref{sec:two-phase}.

\paragraph{Trainable parameters.}
Let \(\mathcal{L}_{\mathrm{MLP}}\) denote the set of adapted MLP weight matrices, and let each matrix \(W_l \in \mathbb{R}^{d_{\mathrm{out},l}\times d_{\mathrm{in},l}}\) have singular-value dimension \(n_l=\min(d_{\mathrm{out},l},d_{\mathrm{in},l})\). In Phase~1, the only trainable variables are the SVF scaling vectors \(\{z_l\}_{l\in\mathcal{L}_{\mathrm{MLP}}}\), so the number of trainable parameters is
\[
P_{\mathrm{SVF}}=\sum_{l\in\mathcal{L}_{\mathrm{MLP}}} n_l.
\]
In Phase~2, PALoRA trains only the LoRA matrices \(\{B_l,A_l\}_{l\in\mathcal{L}_{\mathrm{MLP}}}\). For LoRA rank \(r\), the corresponding trainable parameter count is
\[
P_{\mathrm{LoRA}}=\sum_{l\in\mathcal{L}_{\mathrm{MLP}}} r\,(d_{\mathrm{out},l}+d_{\mathrm{in},l}),
\]
which is the same parameterization as standard LoRA on the same target modules. Thus, PALoRA does not increase the number of trainable Phase~2 parameters relative to a matched LoRA baseline; its additional complexity comes from the frozen SVF expert and the orthogonality constraint.

\paragraph{Additional stored state.}
Relative to standard LoRA, PALoRA requires three additional objects. First, it stores the frozen SVF scaling vectors \(z_l\) obtained in Phase~1. Second, for each adapted layer \(l\), it stores the critical index set
\[
\mathcal{I}_{\mathrm{crit},l}
=
\operatorname{TopK}_{i}\bigl(|z_{l,i}-1|\bigr),
\]
as defined in the main text. Third, it stores the singular vectors associated with these selected indices, since the Phase~2 orthogonality loss is computed by projecting the LoRA update onto the corresponding protected directions. If the protected subspace size is \(k\) for every layer, then the dominant additional storage is linear in the total number of stored singular vectors, namely on the order of \(\sum_l k\,d_{\mathrm{out},l}\) for the left-vector formulation used in this paper. By contrast, the storage of \(z_l\) and \(\mathcal{I}_{\mathrm{crit},l}\) is comparatively small.

\paragraph{Computational overhead.}
PALoRA introduces overhead in two places. The first is an offline cost incurred before Phase~2, consisting of the singular value decomposition of the adapted MLP weights together with the training of the SVF expert in Phase~1. This cost is paid once per skill expert and is independent of the subsequent number of Phase~2 knowledge-injection runs that reuse the same frozen probe.

The second source of overhead appears during Phase~2 training. In addition to the standard supervised fine-tuning objective, PALoRA evaluates the orthogonality regularizer
\[
\mathcal{L}_{\mathrm{ortho},l}
=
\sum_{i\in\mathcal{I}_{\mathrm{crit},l}}
\left\|(B_lA_l)^\top u_{l,i}\right\|_2^2
\]
for each adapted layer. The resulting per-step overhead therefore depends primarily on the number of protected directions and on the dimensions of the adapted MLP layers, rather than on any increase in trainable parameter count.

\paragraph{Inference-time considerations.}
The orthogonality term is used only during optimization in Phase~2 and does not introduce an additional objective at inference time. After training, deployment uses the adapted model parameters in the same way as other LoRA-style methods. Consequently, the practical inference footprint is determined by the final adapted weights rather than by the training-time regularization procedure.

\paragraph{Interpretation.}
This decomposition clarifies the cost profile of PALoRA. Phase~1 adds a one-time preprocessing stage, while Phase~2 preserves the same LoRA parameterization as a matched baseline and adds only a structured regularization term tied to the protected skill subspace. The main trade-off is therefore not an increase in trainable capacity, but the extra storage and computation required to identify and preserve the SVF-guided singular directions.

\section{Datasets, Training, and Evaluation Protocol}
\label{sec:appendix_datasets_protocol}

This section summarizes the datasets used in PALoRA and clarifies their roles in the two-phase training and evaluation pipeline. We distinguish between \emph{skill} datasets, which define the capabilities to be preserved and are used to train the SVF experts in Phase~1, and \emph{knowledge} data, which provide the factual supervision used for LoRA-based knowledge injection in Phase~2. This separation is central to PALoRA: the skill data identify the subspace to protect, while the knowledge data define the new information to be acquired.

\subsection{Overview of dataset roles}

In our experiments, GSM8K, MBPP, and AI2-ARC serve as skill datasets. Each of them is used to train a separate SVF expert in Phase~1, and each reappears after Phase~2 only as an evaluation benchmark for measuring skill preservation. They are therefore not used to optimize the knowledge-injection adapter.

The knowledge-injection data consist of TriviaQA-style factual question--answer pairs labeled according to whether the frozen base model can already answer them reliably. During Phase~2, these examples are used to supervise factual adaptation, whereas final knowledge recall is measured on the subset of facts labeled \texttt{Unknown}. Items labeled \texttt{HighlyKnown} are not the primary target of factual acquisition, but serve as a control for general factual stability.

\subsection{Skill datasets}

\subsubsection{GSM8K}

GSM8K~\cite{cobbe2021trainingverifierssolvemath} is a benchmark of grade-school mathematical word problems and serves as our mathematical reasoning task. In Phase~1, we train a dedicated SVF expert on the GSM8K training split, using a lightweight held-out partition for model selection during the reinforcement-learning loop. GSM8K is not used during Phase~2 optimization.

After knowledge injection, skill preservation is evaluated with final-answer accuracy on the GSM8K benchmark. The model generates a free-form solution, and the prediction is scored by extracting the final numerical answer from the completion and comparing it against the ground-truth value. This metric isolates whether mathematical reasoning remains intact after factual updating.

\subsubsection{MBPP}

MBPP~\cite{austin2021programsynthesislargelanguage} is a benchmark for Python code generation from natural-language problem descriptions and serves as our coding-skill task. In Phase~1, we train an SVF expert on a shuffled subset of MBPP problems, with disjoint training and validation partitions used for learning and checkpoint selection. MBPP is not used in Phase~2 training.

After adaptation, we evaluate skill preservation with \texttt{pass@1} under execution-based evaluation. For each problem, the model generates a single program, which is then executed against the corresponding unit tests. A sample is marked correct only if the generated code passes the test suite, so the metric reflects functional program synthesis rather than surface-level similarity.

\subsubsection{AI2-ARC}

We use the AI2-ARC benchmark~\cite{clark2018thinksolvedquestionanswering} as our science reasoning task. In Phase~1, the SVF expert is trained on ARC-Easy, again using a held-out partition for model selection. Phase~2 does not optimize on ARC data.

For evaluation, we report multiple-choice exact-match accuracy. Questions are presented with lettered answer options, and a prediction is counted as correct if the model selects the correct option after generation. This metric measures whether the scientific reasoning ability associated with the protected SVF subspace is preserved after factual adaptation.

\subsection{Knowledge-injection dataset}

\subsubsection{Unknown facts}

Knowledge injection is performed using TriviaQA-style factual question--answer pairs that are labeled according to whether the frozen base model can already answer them reliably~\cite{pletenev2025knowledge}. The primary target of adaptation is the subset labeled \texttt{Unknown}, which contains facts that are not robustly recalled before fine-tuning. These examples define the factual content to be injected during Phase~2.

For each experiment, we construct a knowledge-injection set by sampling a fixed number of \texttt{Unknown} items. In the main paper, we consider subsets of size 100, 250, 500, and 1000. Each subset defines an independent Phase~2 run, so the factual load can be varied while keeping the underlying protocol unchanged.

\subsubsection{HighlyKnown facts}

In addition to the Unknown subset, we include examples labeled \texttt{HighlyKnown}, corresponding to facts that the base model already answers reliably before adaptation. During Phase~2, these examples are mixed with the Unknown facts at a ratio of \(1{:}3\) (one Unknown fact for three HighlyKnown facts). We adopt this mixture following the setting of Pletenev et al.~\cite{pletenev2025knowledge}, with the goal of coupling factual injection with continued exposure to previously stable knowledge.

This design reflects the stability--plasticity trade-off studied in the main paper. Unknown facts define the new information to be injected, while the accompanying HighlyKnown examples act as a retention signal that helps preserve already acquired knowledge and reduces the risk of catastrophic forgetting during adaptation. At evaluation time, the Unknown subset remains the primary target for measuring factual recall, since it captures whether the model successfully learns the injected facts.

\subsection{Training protocol by phase}

\subsubsection{Phase~1: skill expert training}

Phase~1 trains one SVF expert per skill dataset. Concretely, GSM8K defines the mathematical reasoning expert, MBPP defines the coding expert, and AI2-ARC defines the science reasoning expert. In each case, the trainable parameters are the SVF scaling vectors applied to the singular values of the frozen MLP weight matrices, while the singular vectors themselves remain fixed to those of the pre-trained model.

Following Transformer-Squared~\cite{sun2025transformer}, skill specialization is learned with reinforcement learning rather than supervised imitation of reference completions. For each task, the model generates outputs on skill prompts, and a task-specific evaluator assigns a scalar reward based on correctness. The SVF parameters are then updated using this outcome-based signal, so that the learned scaling vector reflects the singular components most strongly engaged by the target skill.

After training, the best-performing SVF expert is frozen and used as the geometric probe in PALoRA. Its scaling vector defines the protected index set \(\mathcal{I}_{\mathrm{crit}}\), which determines the singular directions that should be preserved during factual adaptation.

\subsubsection{Phase~2: knowledge injection}

In Phase~2, the base model and the Phase~1 SVF expert are kept frozen, and only a LoRA adapter is trained on the factual corpus. For a given run, the protected subspace is defined by the SVF expert associated with a single target skill, so each Phase~2 experiment corresponds to a particular \((\text{skill}, |\mathcal{U}|)\) configuration.

The training objective combines the standard supervised fine-tuning loss on factual question--answer pairs with the PALoRA orthogonality penalty. The supervised term encourages the adapter to produce the target factual answer, while the orthogonality term prevents the LoRA update from aligning with the singular directions identified as skill-critical by the frozen SVF expert. Baseline methods use the same factual supervision but differ in how, or whether, they constrain the adaptation update.

\paragraph{Answer-only supervision.}
Each factual example is presented as a question followed by its target answer. The language-model loss is computed only on the answer span, while the question tokens are masked. The model therefore conditions on the full question at training time, but the optimization target is restricted to producing the desired factual completion rather than reconstructing the prompt itself.

\subsection{Evaluation protocol}

\subsubsection{Skill evaluation}

After Phase~2, we evaluate the adapted model on the skill benchmark corresponding to the SVF expert used in that run. GSM8K is scored with final-answer accuracy, MBPP with \texttt{pass@1}, and AI2-ARC with multiple-choice exact-match accuracy. These metrics quantify \emph{skill preservation}, that is, the extent to which the target reasoning capability survives factual adaptation.

\subsubsection{Knowledge evaluation}

We additionally evaluate the adapted model on the \texttt{Unknown} portion of the factual dataset. At test time, the model is prompted with the question only, and a prediction is counted as correct if the generated answer matches one of the accepted aliases for the target fact. The resulting recall measures the fraction of previously unknown facts that are answered correctly after adaptation.

We also evaluate the \texttt{HighlyKnown} subset with the same scoring rule. This provides a control for factual drift by checking whether the model's behavior on previously reliable answers remains stable after knowledge injection.

\subsubsection{Protocol consistency across methods}

All methods in our comparison use the same factual datasets, the same skill benchmarks, and the same evaluation rules for a given model and fact-set size. Differences between methods therefore arise from the adaptation mechanism itself rather than from changes in data, prompting, or scoring. This common protocol is essential for isolating the effect of SVF-guided orthogonal protection in PALoRA.

\section{Baseline Adaptation and Fairness Protocol}
\label{sec:baseline_fairness}

To ensure that the baseline comparisons are directly meaningful, we adapt each baseline to the same data interface, optimization setting, and evaluation protocol used for PALoRA, so that performance differences can be attributed to the adaptation principle itself rather than to mismatched experimental conditions.

\paragraph{Common protocol.}
All methods use the same base checkpoints, the same knowledge-injection datasets, and the same train--validation--test protocol. For knowledge injection, we use the same answer-only supervised objective across methods, masking the prompt tokens and optimizing only the target answer tokens. Across both model families, we match the adapter configuration used in our main experiments, including adapter rank \(r=16\), scaling \(\alpha=32\), dropout \(0.05\), and batch size \(8\). The optimization schedule is model-specific: for Llama-3.1-8B-Instruct, we use a learning rate of \(2\times10^{-4}\) and train for \(10\) epochs, whereas for Mistral-7B-Instruct-v0.3, we use a learning rate of \(2\times10^{-5}\) and train for \(5\) epochs. Evaluation is likewise standardized: knowledge performance is measured with the same Unknown and HighlyKnown splits, and skill preservation is assessed with the same task-specific evaluation pipeline for GSM8K, MBPP, and AI2-ARC.

\paragraph{LoRA-only.}
The LoRA-only baseline uses the same knowledge-injection setting as PALoRA, but removes the structural orthogonality constraint. It is therefore trained as an unconstrained adapter on the same target MLP projections, with the same adapter capacity and the same optimization schedule as PALoRA. This baseline isolates the effect of the constraint itself, while keeping the underlying adaptation budget fixed.

\paragraph{OPLoRA.}
For OPLoRA, we compute the protected subspace from the same pre-trained model weights and apply the method on the same MLP layers used by PALoRA. Following the original OPLoRA setting, we use a protected subspace size of \(k=128\). We otherwise keep the adapter rank, training objective, optimizer, and training schedule matched to our default setting. As a result, the comparison focuses on the protection mechanism: OPLoRA constrains updates with respect to top singular directions of the pre-trained weights, whereas PALoRA constrains LoRA using the SVF-identified skill-relevant subspace.

\paragraph{CorDA.}
We implement CorDA in Knowledge-Preserved Mode and evaluate it under the same downstream knowledge-injection protocol as the other baselines. Its calibration statistics are computed from skill-specific training data using 256 calibration samples. The subsequent fine-tuning stage follows the same optimizer family, training budget, and evaluation procedure as in the other LoRA-based comparisons, so the comparison remains aligned at the level of data and protocol.

For evaluation, we convert the trained CorDA adapter into a standard LoRA-form adapter so that it can be assessed through the same inference and reporting pipeline as the other methods. This conversion changes the effective adapter rank from 16 in the original CorDA parameterization to 32 in the converted LoRA form. We report CorDA under this converted setting to preserve a common evaluation interface while retaining its knowledge-preserved initialization mechanism.

\paragraph{Fairness of comparison.}
Under this protocol, all competing methods are aligned in model family, data, optimization budget, target layers, and evaluation metrics. The reported differences therefore reflect the underlying adaptation strategy---unconstrained LoRA, top-component protection, context-preserving initialization, or SVF-guided orthogonal protection---rather than artifacts of inconsistent training or evaluation pipelines.

\section{Additional SVF Adapter Visualizations}
\label{app:svf_visualizations}

We provide the same SVF visualizations as in Figure~\ref{fig:distributed_evidence} for all skill experts across both model families. For each SVF expert, we show (i) the sorted scaling magnitudes \( |z_i^* - 1| \) across MLP layers and (ii) the top-64 SVF-selected components, sorted by \( |z_i^* - 1| \) from largest to smallest.

\begin{figure}[!htbp]
    \centering
    \begin{minipage}{0.47\textwidth}
        \centering
        \includegraphics[width=\textwidth]{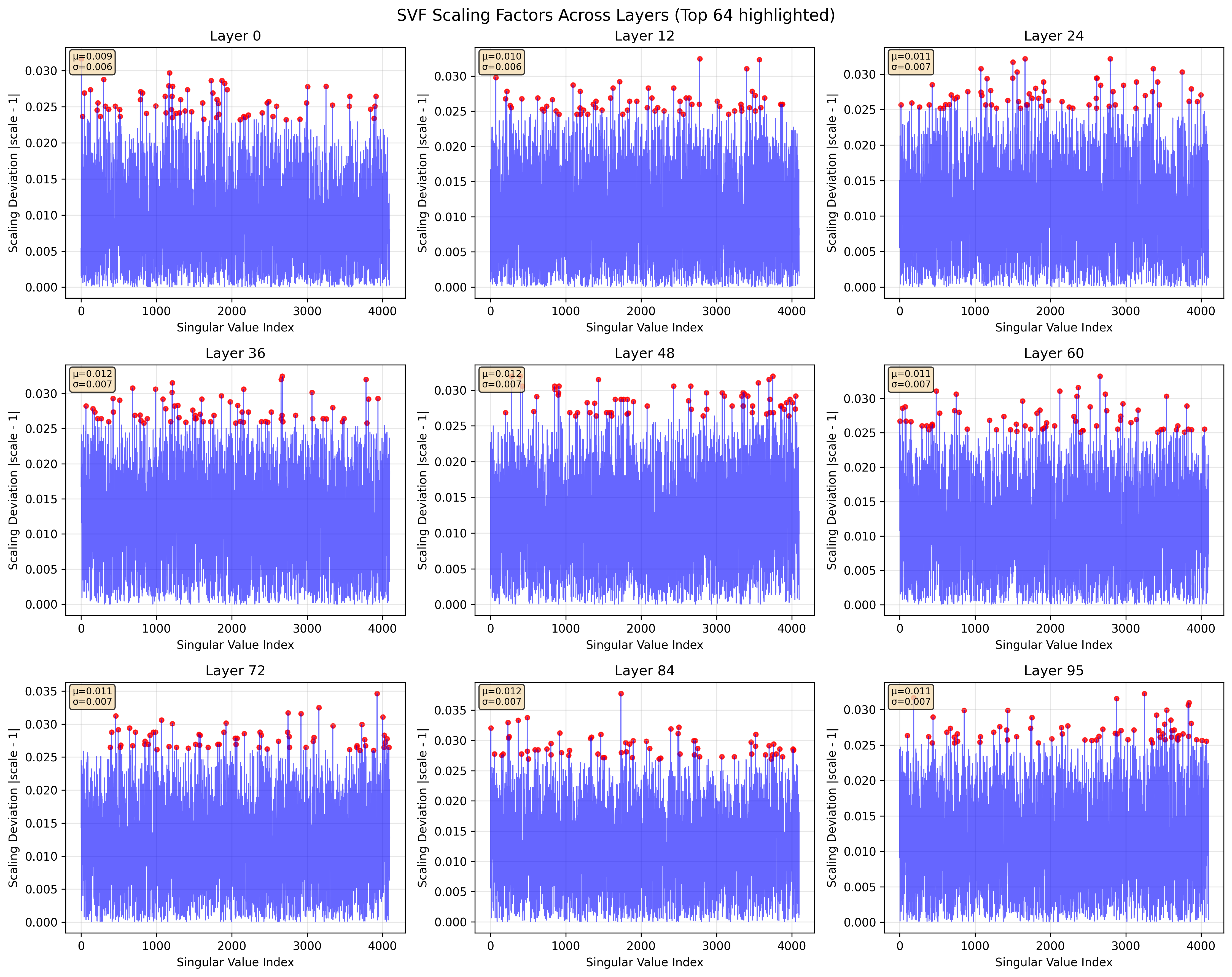}
        \small (a) SVF scaling magnitudes \( |z_i^*-1| \) over singular-value index.
    \end{minipage}
    \hfill
    \begin{minipage}{0.47\textwidth}
        \centering
        \includegraphics[width=\textwidth]{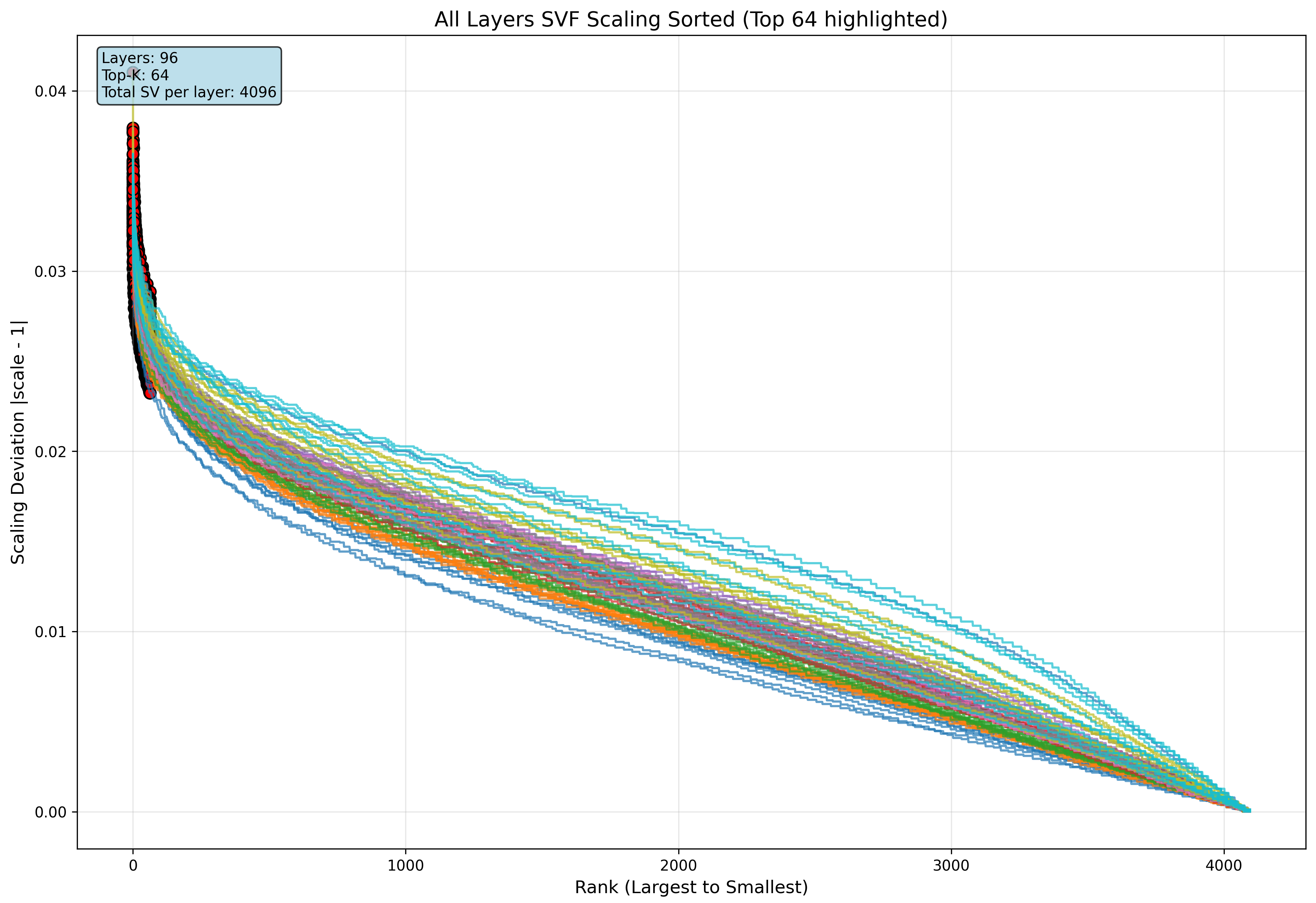}
        \small (b) Top-64 SVF scaling magnitudes \( |z_i^*-1| \), sorted from largest to smallest.
    \end{minipage}
    \caption{SVF visualization for the Llama-3.1-8B-Instruct GSM8K expert.}
    \label{fig:appendix_llama_gsm8k}
\end{figure}

\begin{figure}[!htbp]
    \centering
    \begin{minipage}{0.47\textwidth}
        \centering
        \includegraphics[width=\textwidth]{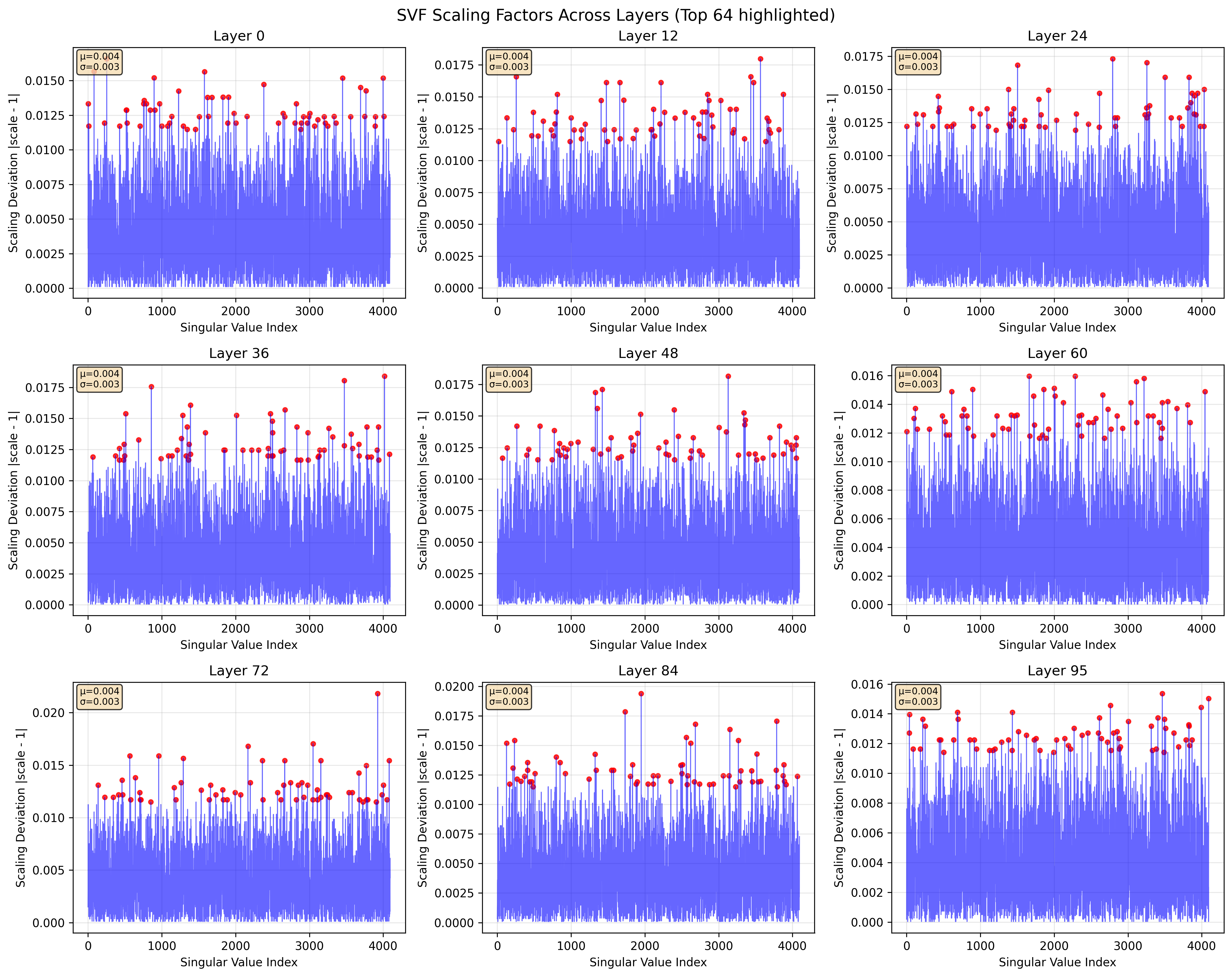}
        \small (a) SVF scaling magnitudes \( |z_i^*-1| \) over singular-value index.
    \end{minipage}
    \hfill
    \begin{minipage}{0.47\textwidth}
        \centering
        \includegraphics[width=\textwidth]{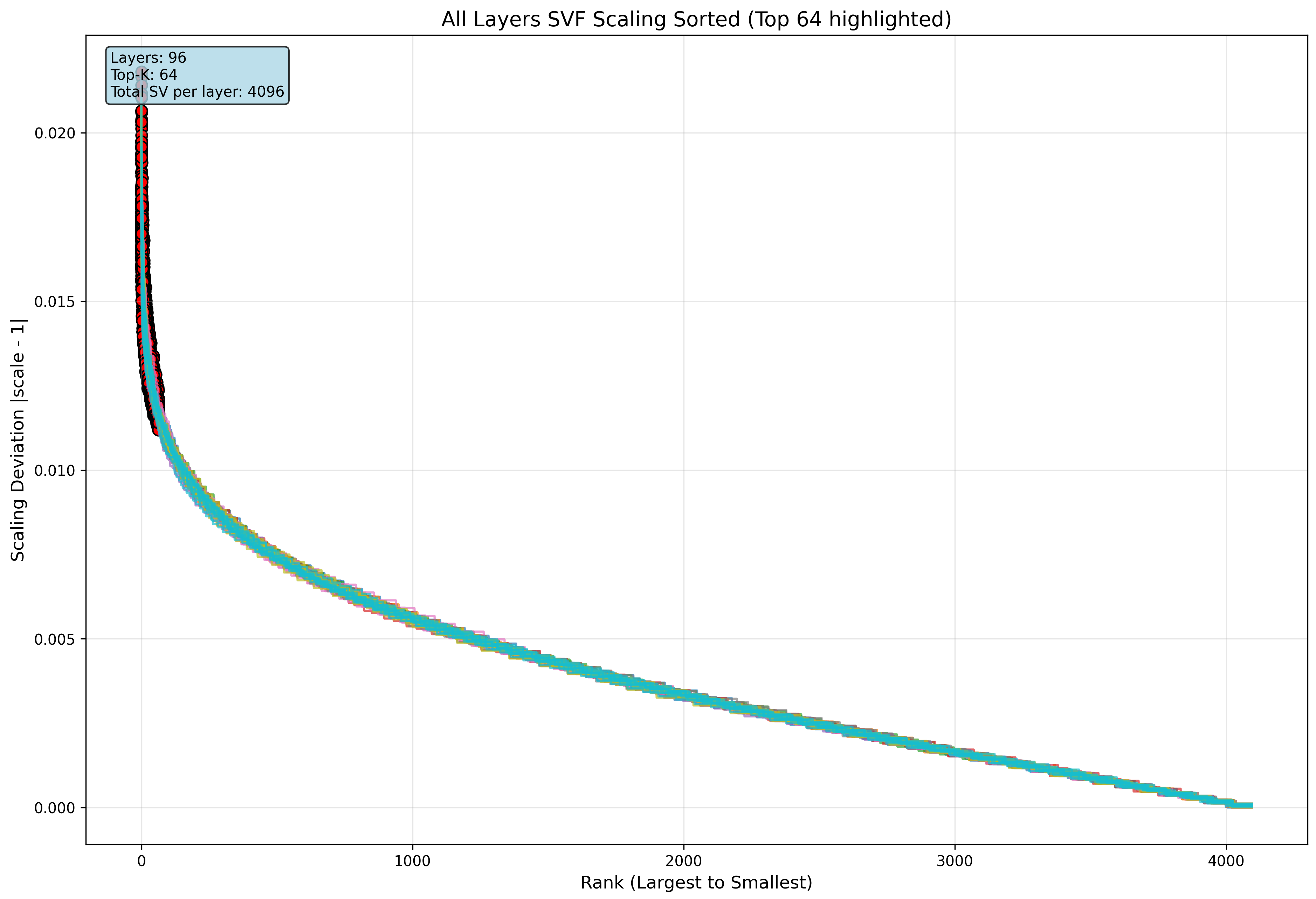}
        \small (b) Top-64 SVF scaling magnitudes \( |z_i^*-1| \), sorted from largest to smallest.
    \end{minipage}
    \caption{SVF visualization for the Llama-3.1-8B-Instruct MBPP expert.}
    \label{fig:appendix_llama_mbpp}
\end{figure}

\begin{figure}[!htbp]
    \centering
    \begin{minipage}{0.47\textwidth}
        \centering
        \includegraphics[width=\textwidth]{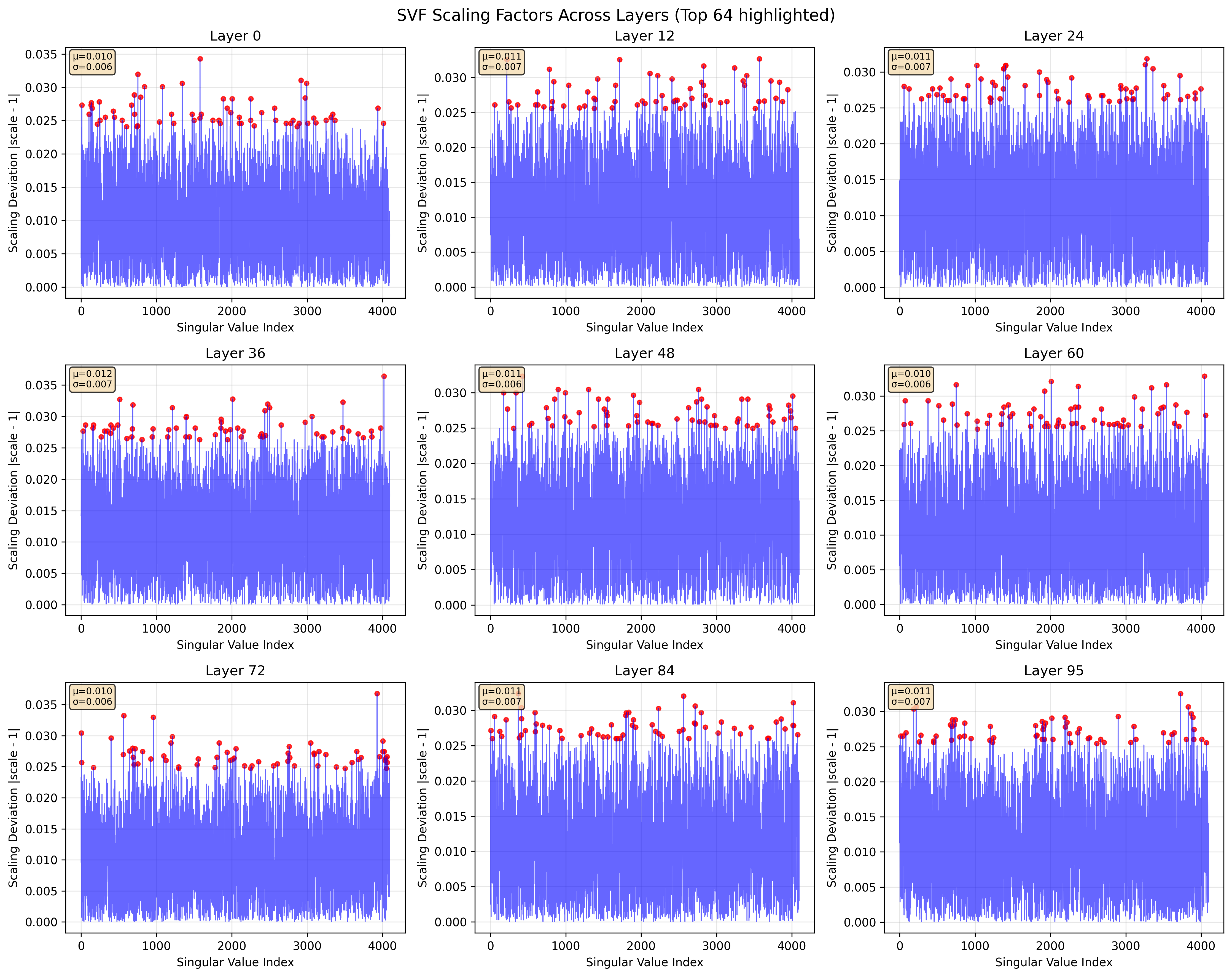}
        \small (a) SVF scaling magnitudes \( |z_i^*-1| \) over singular-value index.
    \end{minipage}
    \hfill
    \begin{minipage}{0.47\textwidth}
        \centering
        \includegraphics[width=\textwidth]{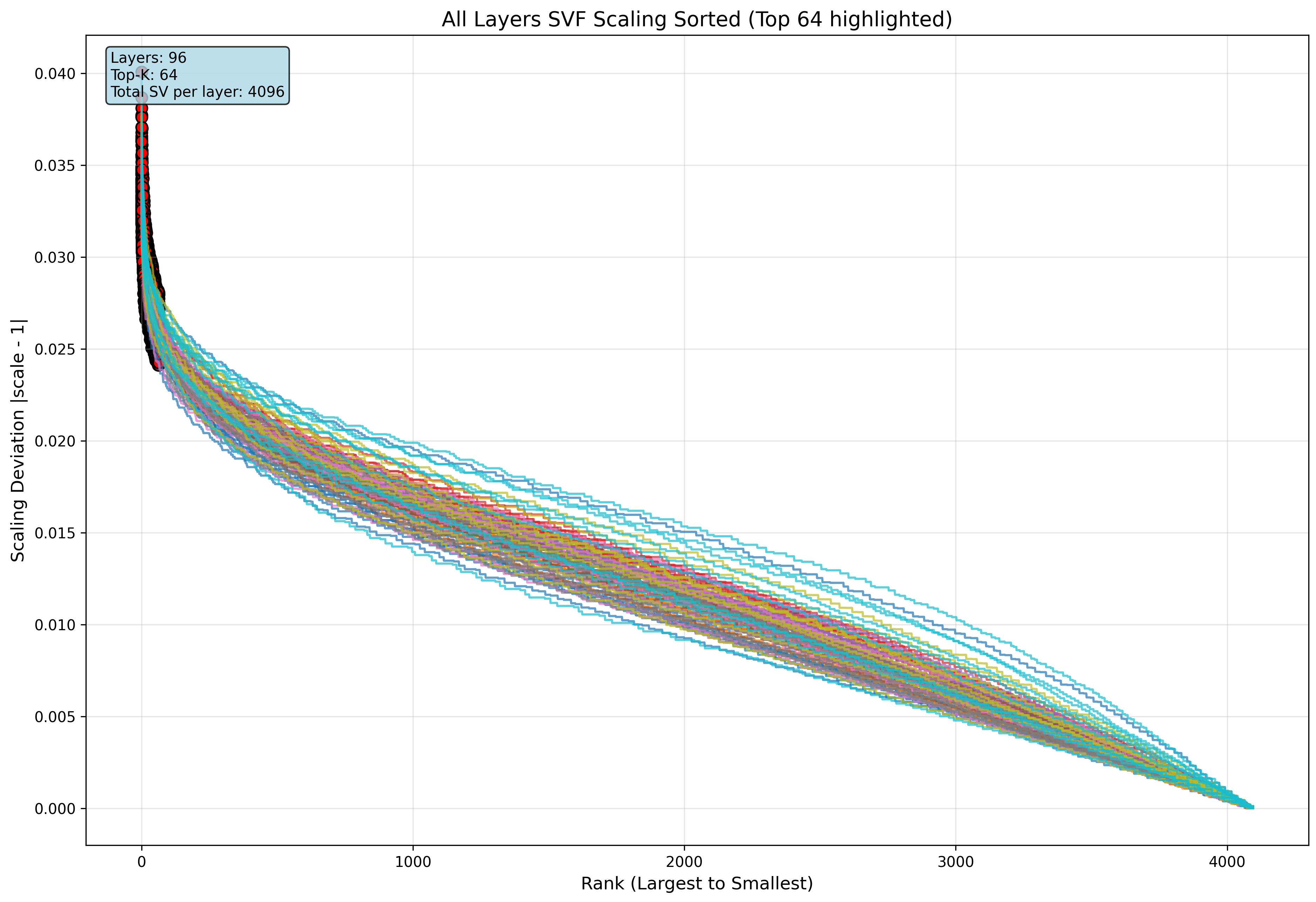}
        \small (b) Top-64 SVF scaling magnitudes \( |z_i^*-1| \), sorted from largest to smallest.
    \end{minipage}
    \caption{SVF visualization for the Llama-3.1-8B-Instruct AI2-ARC expert.}
    \label{fig:appendix_llama_arc}
\end{figure}

\begin{figure}[!htbp]
    \centering
    \begin{minipage}{0.47\textwidth}
        \centering
        \includegraphics[width=\textwidth]{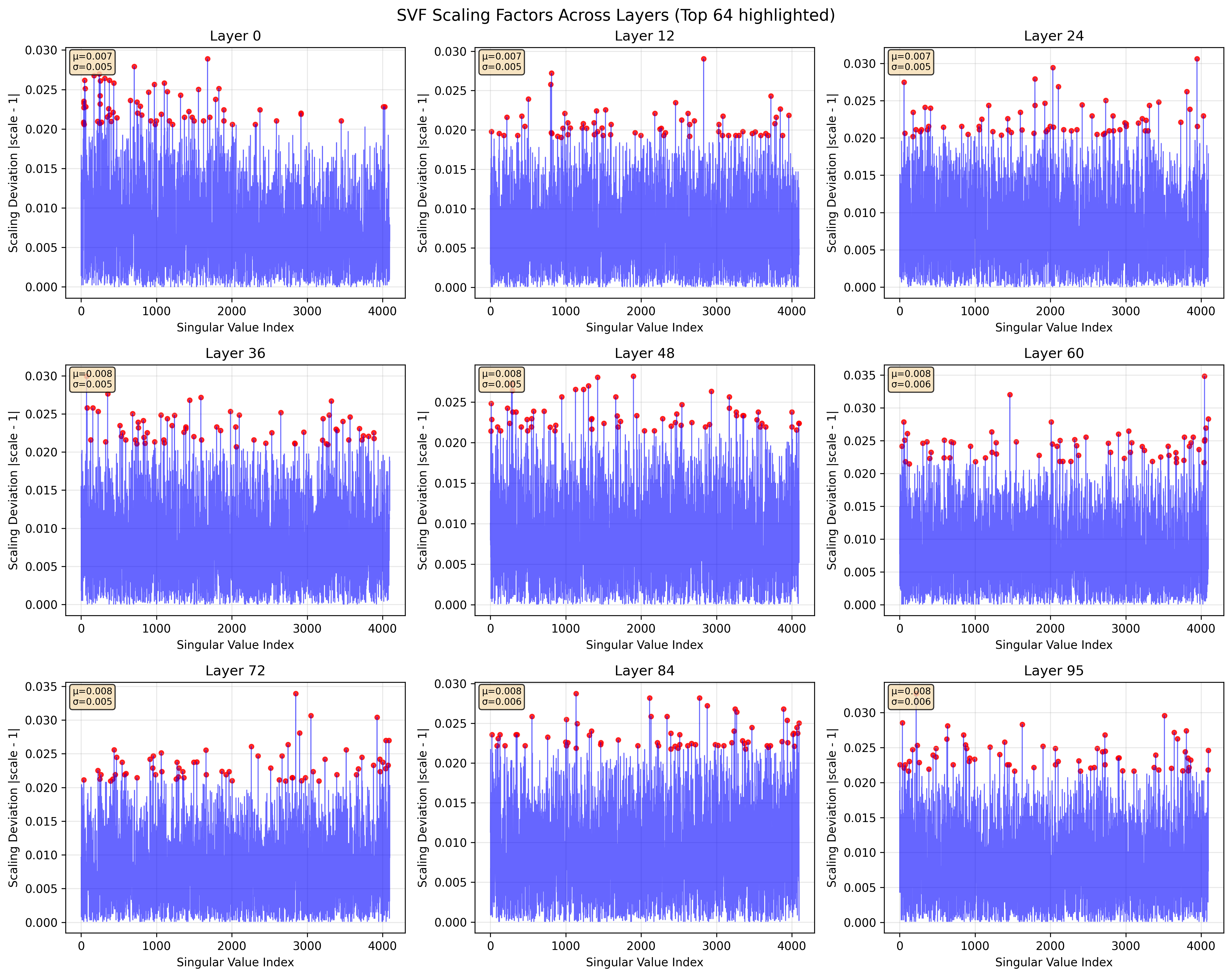}
        \small (a) SVF scaling magnitudes \( |z_i^*-1| \) over singular-value index.
    \end{minipage}
    \hfill
    \begin{minipage}{0.47\textwidth}
        \centering
        \includegraphics[width=\textwidth]{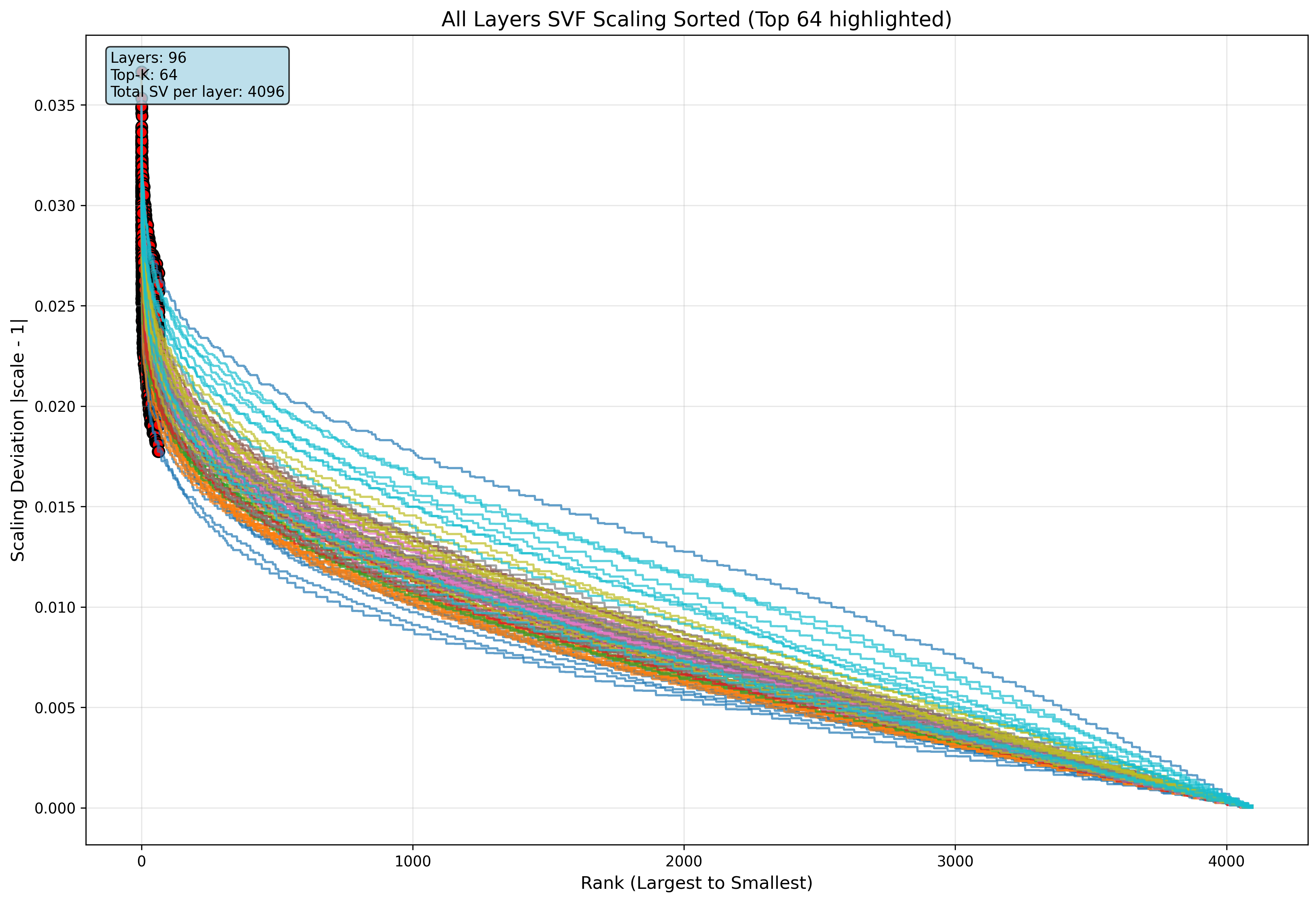}
        \small (b) Top-64 SVF scaling magnitudes \( |z_i^*-1| \), sorted from largest to smallest.
    \end{minipage}
    \caption{SVF visualization for the Mistral-7B-Instruct-v0.3 GSM8K expert.}
    \label{fig:appendix_mistral_gsm8k}
\end{figure}

\begin{figure}[!htbp]
    \centering
    \begin{minipage}{0.47\textwidth}
        \centering
        \includegraphics[width=\textwidth]{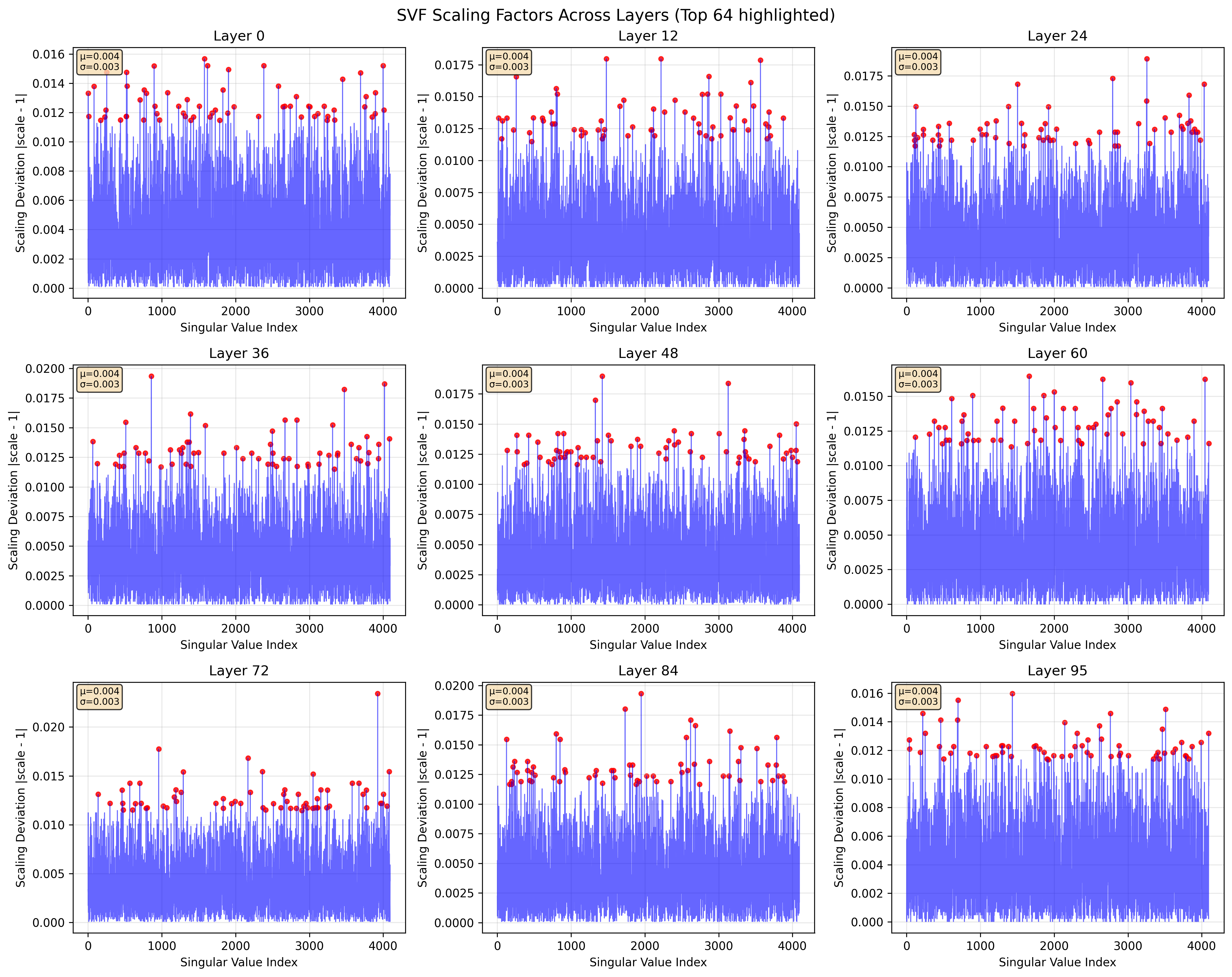}
        \small (a) SVF scaling magnitudes \( |z_i^*-1| \) over singular-value index.
    \end{minipage}
    \hfill
    \begin{minipage}{0.47\textwidth}
        \centering
        \includegraphics[width=\textwidth]{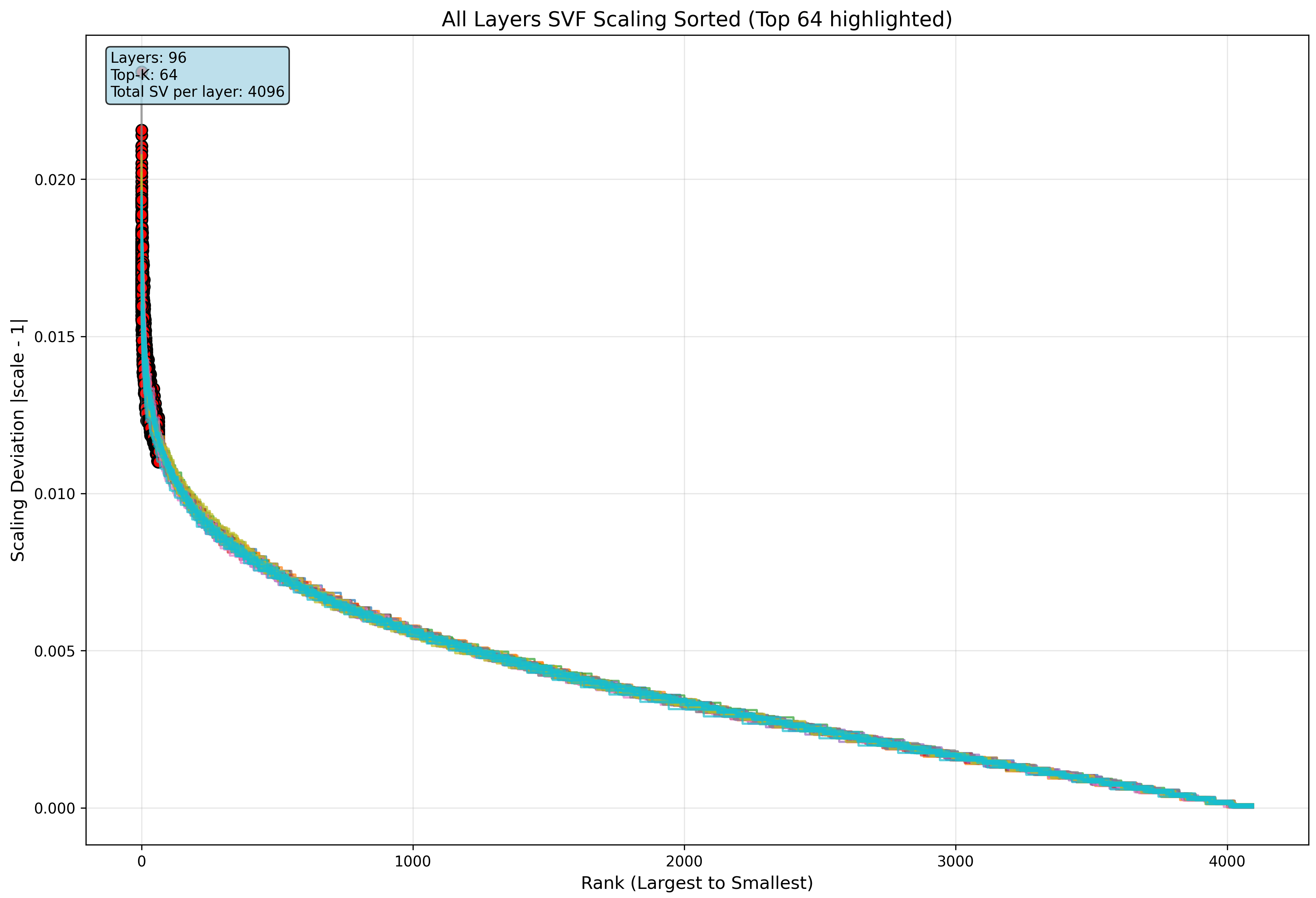}
        \small (b) Top-64 SVF scaling magnitudes \( |z_i^*-1| \), sorted from largest to smallest.
    \end{minipage}
    \caption{SVF visualization for the Mistral-7B-Instruct-v0.3 MBPP expert.}
    \label{fig:appendix_mistral_mbpp}
\end{figure}

\begin{figure}[!htbp]
    \centering
    \begin{minipage}{0.47\textwidth}
        \centering
        \includegraphics[width=\textwidth]{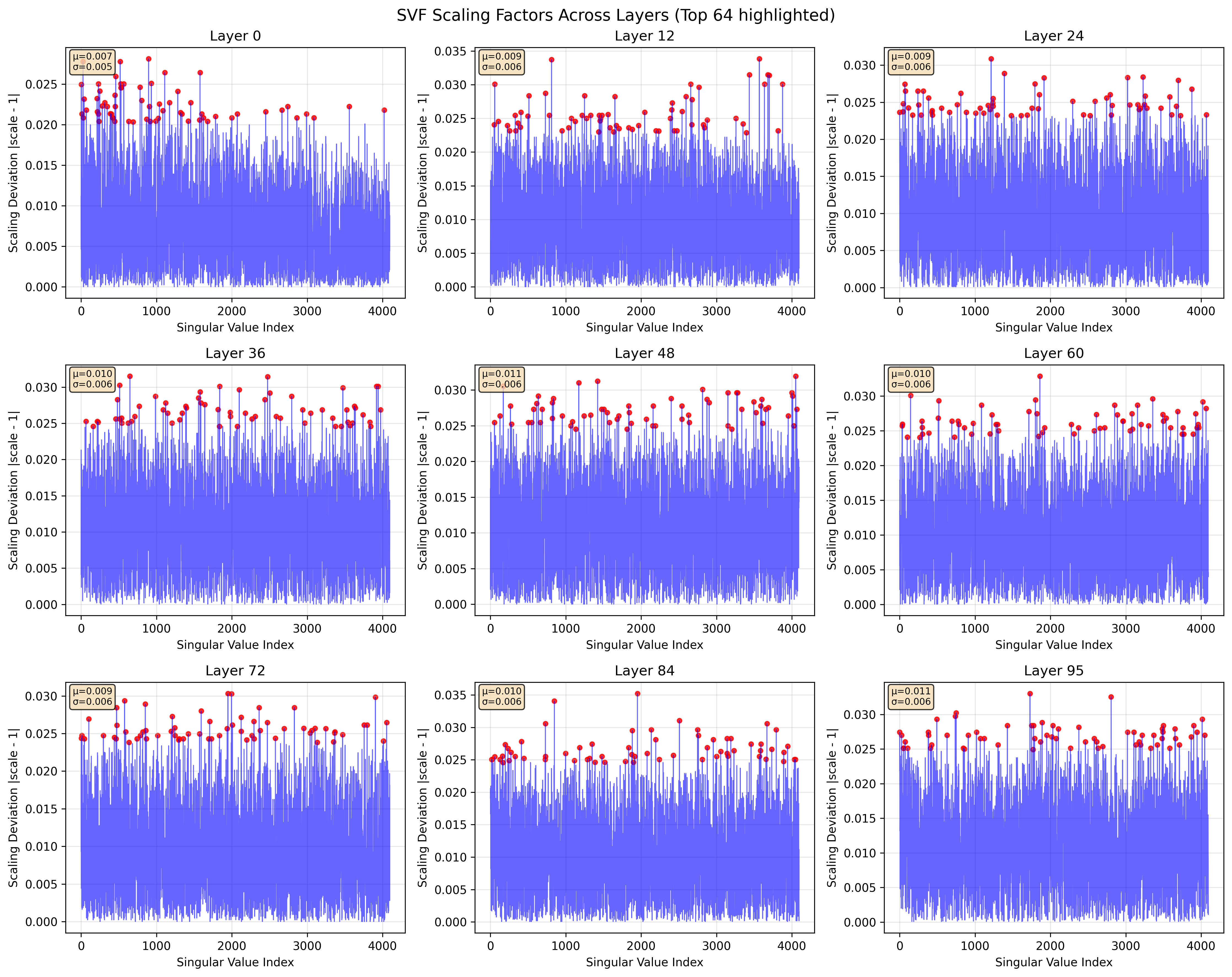}
        \small (a) SVF scaling magnitudes \( |z_i^*-1| \) over singular-value index.
    \end{minipage}
    \hfill
    \begin{minipage}{0.47\textwidth}
        \centering
        \includegraphics[width=\textwidth]{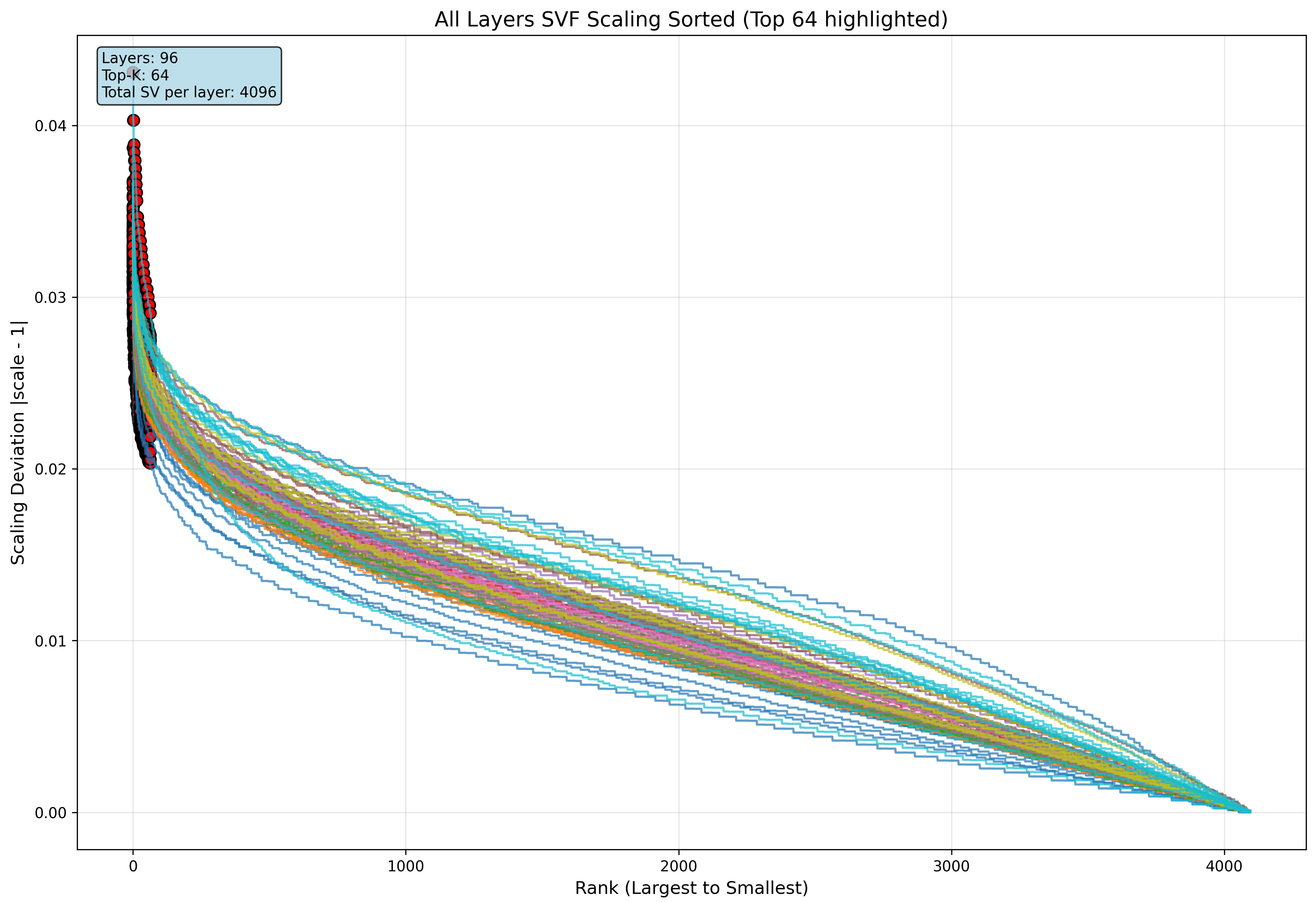}
        \small (b) Top-64 SVF scaling magnitudes \( |z_i^*-1| \), sorted from largest to smallest.
    \end{minipage}
    \caption{SVF visualization for the Mistral-7B-Instruct-v0.3 AI2-ARC expert.}
    \label{fig:appendix_mistral_arc}
\end{figure}

\end{document}